\documentclass[letterpaper]{article}
\usepackage{proceed2e}
\usepackage{times}
\usepackage{pdfpages}
\usepackage{helvet}
\usepackage{courier}
\usepackage{algorithm}
\usepackage{algorithmic}
\usepackage{mathtools, nccmath}
\usepackage{amsmath}
\usepackage{amsfonts}
\usepackage{bbm}
\usepackage{graphicx} 
\usepackage{enumitem}
\usepackage{float}
\usepackage{mathtools}
\usepackage{enumitem}
\usepackage{amsmath}
\usepackage[utf8]{inputenc} % allow utf-8 input
\usepackage[T1]{fontenc}    % use 8-bit T1 fonts
\usepackage{hyperref}       % hyperlinks
\usepackage{url}            % simple URL typesetting
\usepackage{booktabs}       % professional-quality tables
\usepackage{amsfonts}       % blackboard math symbols
\usepackage{nicefrac}       % compact symbols for 1/2, etc.
\usepackage{microtype}      % microtypography
\usepackage{graphicx}
\usepackage{caption}
\usepackage{amsthm}
\usepackage{nicefrac}
\usepackage{subcaption}
\newtheorem{theorem}{Theorem}[section]
\newtheorem{lemma}[theorem]{Lemma}

\usepackage{enumitem}

\newcommand{\bs}{\boldsymbol}


\usepackage{lipsum}

\usepackage{lipsum}
\usepackage{mathtools}
\begin{document}

%\usepackage{natbib}
%\usepackage{subcaption}
%\usepackage{natbib}
%\usepackage[square,sort,comma,numbers]{natbib}
% \title{Scaling Importance Weighting Covariate Shift Methods to Higher Dimensions}

% % \author{
% %   Fulton Wang\\
% %   EECS Department\\
% %   MIT\\
% %   Cambridge, MA 02139 \\
% %   \texttt{fultonw@mit.edu} \\
% % \And
% % Cynthia Rudin\\
% %   Department of Computer Science\\
% % Duke University\\
% %   Durham, NC 27708 \\
% %   \texttt{cynthia@cs.duke.edu} \\
% % \And
% %   Aaron Fisher\\
% % %  Department of Biostatistics\\
% %  Harvard T.H. Chan \\School of Public Health\\
% %   Boston, MA 02115\\
% %   \texttt{aafisher@hsph.harvard.edu} \\
% % }

\renewcommand*{\proofname}{Pf.}

\makeatletter
\renewenvironment{proof}[1][\proofname]{\par
%  \pushQED{\qed}% <--- remove the QED business
  \normalfont \topsep0\p@\@plus0\p@\relax
  \trivlist
  \item[\hskip\labelsep
        \itshape
    #1\@addpunct{.}]\ignorespaces
}{%
%  \popQED% <--- remove the QED business
  \endtrivlist\@endpefalse
}
\renewcommand\qedhere{} % to ensure code portability
\makeatother
%\begin{document}
% \nipsfinalcopy is no longer used

% You may provide any keywords that you 
% find helpful for describing your paper; these are used to populate 
% the "keywords" metadata in the PDF but will not be shown in the document

%\maketitle

\title{\hspace*{-22.9pt}Extreme Dimension Reduction for Handling Covariate Shift\hspace*{-22.9pt}}

\author{Ben Trovato}
%\authornote{Dr.~Trovato insisted his name be first.}
%\orcid{1234-5678-9012}

\author{G.K.M. Tobin}

\author{ 
{\bf Fulton Wang}  \\
MIT          \\
fultonw@mit.edu \\
\And
{\bf Cynthia Rudin}   \\
Duke University \\
cynthia@cs.duke.edu    \\
}

%\authornote{The secretary disavows any knowledge of this author's actions.}

% The default list of authors is too long for headers.

\maketitle

 \begin{abstract}
In the covariate shift learning scenario, the training and test covariate distributions differ, so that a predictor's average loss over the training and test distributions also differ.  In this work, we explore the potential of extreme dimension reduction, i.e. to very low dimensions, in improving the performance of importance weighting methods for handling covariate shift, which fail in high dimensions due to potentially high train/test covariate divergence and the inability to accurately estimate the requisite density ratios.  We first formulate and solve a problem optimizing over linear subspaces a combination of their predictive utility and train/test divergence within.  Applying it to simulated and real data, we show extreme dimension reduction helps sometimes but not always, due to a bias introduced by dimension reduction.% in the covariate shift scenario.
 \end{abstract}
%to encourage the subspace to retain information related to the outcome of interest

\section{Introduction}
Often, the population for which one has labelled training data differs from the population one needs to make predictions for. 
%For example, the population that received a medical treatment likely differs from those eligible for it.  
In handling this discrepancy one often assumes the training and test domain have identical conditional outcome distributions ($P^{tr}_{Y|X} = P^{te}_{Y|X})$ but different marginal covariate distributions ($P^{tr}_{X} \neq P^{te}_X$).  The popular importance weighting (IW) approach \cite{shimodaira2000improving, gretton2009covariate,sugiyama2008direct, bickel2009discriminative} to solving this \emph{covariate shift} problem fails in high dimensions, in which the density ratios $P^{tr}_X(x) /P^{te}_X(x)$ used to reweight the training loss are unreliably estimated and concentrated on a small number of training samples \cite{cortes2010learning}, leading to high variance estimates of a predictor's expected test loss, the quantity to be minimized over predictors.  In this work, we propose a extreme linear dimension reduction preprocessing step, i.e. to very low dimensions ($\leq 4$ and often 1 in our experiments), to improve IW methods' performance on high dimensional data.

We propose this dimension reduction be \emph{extreme} to maximally reduce the IW estimator variance due to density ratio estimation error.  We propose this dimension reduction be \emph{linear} so that the downstream IW method, which we will assume to learn linear models, would still learn models linear in the original covariates, thus retaining interpretability.  We propose that \emph{some} dimension reduction should be done, so that we might uncover a subspace within which the train and test covariate distributions were not too different.  This would afford the downstream IW estimator a larger effective sample size and thus lower variance.

It is not clear a priori that extreme dimension reduction would be computationally and experimentally favorable.  Firstly, given the extreme low dimension of the subspaces we consider, in selecting one, we wish to estimate subspaces' true predictive utility: the minimum achievable test domain loss of a predictor acting on it.  However, estimating this utility for just a \emph{single} candidate subspace requires solving a convex optimization problem.  We must confront this potentically impractical computational problem; a cheaper but less accurate proxy for that utility is unlikely to suffice in extremely low dimensions.
% it is unlikely a surrogate measure of its predictive utility would accurately reflect the subspace's true predictive utility: the minimum achievable loss of a predictor acting on it.  Therefore, we explicitly estimate that predictive utility in evaluating a given subspace, which notably requires the solving of a convex optimization problem.
Secondly, as we point out, dimension reduction introduces bias into the downstream IW estimator, so that extreme dimension reduction might add too much bias to actually help.

To test the utility of extreme dimension reduction, we first develop a procedure that maximizes, over projection matrices, estimates of the resulting subspace's aforementioned true predictive utility, regularized by a penalty against subspaces where the downstream IW estimator would have high variance.  This procedure, using techniques from bilevel optimization, is verified to be computationally feasible.  We then study on real data whether it can identify subspaces where the downstream IW method can reliably have low test loss.

We find that on some, but not all real data, extreme dimension reduction can facilitate the learning of accurate predictors for the test domain.  We hypothesize that for such success cases, our procedure returns sufficiently predictive subspaces within which the reduction in downstream IW estimator variance offsets the potential estimator bias.  To better understand this phenomena, we construct a simulated data example where extreme dimension reduction helps, and another one where the introduced bias is crippling.  Finally, reformulating the problem of learning a subgroup-specific model as a covariate shift problem, we demonstrate the benefit of extreme dimension reduction in learning a depression classifier for the college-aged population subgroup.

\section{Background}

\subsection{Covariate Shift Problem \label{sec:cov_shift}}
In the covariate shift problem, one is given $N^{tr}$ labelled training samples $\{x_i^{tr},y_i^{tr}\}$, with $x_i^{tr},y_i^{tr} \sim P^{tr}_{X,Y}$, the training distribution, and $N^{te}$ unlabelled test samples $\{x_i^{te}\}$, with $x_i^{te} \sim P^{te}_X$, the test distribution, with $x_i^{tr},x_i^{te} \in \mathbb{R}^D$.  Importantly, one makes the covariate shift assumption, that $P^{tr}_{Y|X}=P^{te}_{Y|X}$, but $P^{tr}_X\neq P^{te}_X$.  Given a model class $\mathcal{F}$, the covariate shift problem seeks $\operatorname{argmin}_{f\in\mathcal{F}} E_{P^{te}_{X,Y}}[l(f(X),Y)]$, where $l$ is a loss function that we assume to be convex in $f(X)$.

\subsection{Importance Weighted Loss Minimization}
To minimize expected test loss over predictors, noting that $E_{P^{te}_{X,Y}}[l(f(X),Y)]=E_{P^{tr}_{X,Y}}[\tfrac{P^{te}_X(X)}{P^{tr}_X(X)}l(f(X),Y)]$,  past work constructs an unbiased estimator of test loss by forming the empirical expectation version of the latter expectation, and then minimizes it over predictors, adding some regularization:
%minimizes an unbiased estimator of test loss over predictors plus a regularizer:
\begin{align}
&\operatorname{argmin}_{f\in\mathcal{F}}\sum_i \hat{w}_il(f(x_i^{tr},y_i^{tr}) + \Omega(f)\label{eq:naive}
\end{align}
where $\hat{w}_i$ estimates $w_i\coloneqq\tfrac{P^{te}_X(x_i^{tr})}{P^{tr}_X(x_i^{tr})}$.
We note that the ``effective sample size'' of the estimator is $N^{tr}/\sum_i \hat{w}_i^2$ \cite{doucet2001sequential}.

\subsection{Density Ratio Estimation\label{eq:ratio}}
To carry out importance weighting requires estimating the density ratios $\tfrac{P^{te}_X(x_i^{tr})}{P^{tr}_X(x_i^{tr})}$.  A variety of methods \cite{gretton2009covariate, bickel2009discriminative} exist for doing so, but one, least squares importance fitting (LSIF)\cite{kanamori2009least} and its computationally more efficient variant, unconstrained LSIF (uLSIF) stands out because it involves only a (possibly constrained) quadratic program, and admits a cross validation scheme.%  We will use it to estimate density ratios, though other methods, being convex problems, can be handled in our framework.

LSIF assumes the estimated ratio function can be written as $\boldsymbol{\alpha}^T \boldsymbol{\phi}(x)$, where $\boldsymbol{\phi}(x)=\{\phi_m(x)\}$ is a set of $M$ basis functions they choose to be gaussian kernels centered at test points and $\boldsymbol{\alpha}\in \mathbb{R}^K$ is a parameter they aspire to identify by minimizing the expected squared error to the true ratios: $E_{P_X^{tr}}[(\boldsymbol{\alpha}^T \boldsymbol{\phi}(X)-\tfrac{P_X^{te}(X)}{P_X^{tr}(X)})^2] = E_{P^{tr}_X}[(\boldsymbol{\alpha}^T \boldsymbol{\phi}(X))^2] - 2E_{P^{te}_X}[\boldsymbol{\alpha}^T \boldsymbol{\phi}(X)]+C$, where $C$ is a constant.  Substituting empirical expectations, dividing by 2, and regularizing, they obtain the fit value of $\boldsymbol{\alpha}$ as 
\begin{align}
\hspace{-0pt}\boldsymbol{\alpha}^*=\operatorname{argmin}_{\boldsymbol{\alpha}} \tfrac{1}{2} \boldsymbol{\alpha}^T\bs{H} \boldsymbol{\alpha} - \boldsymbol{h}^T\boldsymbol{\alpha} + \gamma \boldsymbol{1}^T\boldsymbol{\alpha}\ \text{subject to}\ \boldsymbol{\alpha}\geq 0, \nonumber
\end{align}
with $\boldsymbol{H}=\tfrac{1}{N^{tr}}\sum_k \boldsymbol{\phi}(x_i^{tr}) \boldsymbol{\phi}(x_i^{tr})^T$, $\boldsymbol{h}=\tfrac{1}{N^{te}}\sum_k\boldsymbol{\phi}(x_i^{te})$, $\gamma$ is the regularization constant, and $\boldsymbol{1}$ is the vector of all 1's.  The estimated ratio function is then $\hat{w}(x)=\bs{\alpha}^{*T}\phi(x)$, with the element-wise positive constraint $\boldsymbol{\alpha}\geq 0$ ensuring $\hat{w}(x)$ is positive.  uLSIF simply removes the positivity constraint on $\bs{\alpha}$ in obtaining $\bs{\alpha}^*$.  With $\bs{\alpha}^T\bs{\phi}(x)$ no longer guranteed to be positive, under uLSIF, the learned ratio function is $\hat{w}(x)=\max(\bs{\alpha}^{*T}\bs{\phi}(x),0)$.

\section{Extreme Dimension Reduction for Importance Weighting}

\subsection{Motivation\label{sec:motivation}}
Our dimension reduction approach is motivated by the following: suppose we will apply a linear projection prior to an IW method for handling covariate shift. Furthermore, suppose (as we will in this work) that the model class of the IW method is the set of linear models.  That is, we first choose subspace dimension $K$, a projection matrix $A\in\mathcal{R}^{K\times D}$, and then minimize over $b\in\mathcal{R}^K$ (with regularization)% and $A$ fixed) 
\begin{align}
\hspace*{-7pt}\hat{L}(b;A,\hat{w}^A)\coloneqq \tfrac{1}{N^{tr}}\sum_i \hat{w}^A_il\left(b^TA^Tx_i^{tr},y_i^{tr} \right), \label{eq:estimator}\ \text{where}\hspace*{-8pt}
\end{align} 
$\hat{w}^A_i$ estimates $w^A_i\coloneqq\tfrac{P_{A^TX}^{te}(A^Tx_i^{tr})}{P_{A^TX}^{tr}(A^Tx_i^{tr})}$ and $\hat{w}^A\coloneqq\{\hat{w}^A_i\}$.%\hspace*{-18pt}

This projection by $A$ can potentially reduce the variance of the downstream IW loss estimator if we choose $A$ to satisfy two criteria.  Firstly, if the dimension of the subspace $A$ projects to is greatly reduced, the density ratio estimates themselves will be of lower variance, being those in a lower dimensional space.  Secondly, we can choose $A$ so the effective sample size of the IW estimator following projection by $A$, $N^{tr}/\sum_i (\hat{w}^A_i)^2$, would be above some threshold.
%if the two domains are similar within the subspace parametrized by $A$, the effective sample size of the downstream IW loss estimator would be relatively large.

%Of course, criteria beyond downstream estimator variance also guide the choice of $A$ by which the data would potentially be projected.  
Of course, projection by $A$ might hurt performance if other criteria are not satisfied.  We do not want projection by $A$ to introduce too much bias in the downstream IW estimator, so that we may not want $A$ to be too low-dimensional.
And of course we want $A$ to have high predictive utility; \begin{align}
U(A)\coloneqq\operatorname{argmin}_{b\in\mathbb{R}^K}E_{P^{te}}[l(b^TA^TX)] \label{eq:U_A}
\end{align}
should be low.

Constructively, we can obtain an $A$ satisfying a combination of those criteria by minimizing over projection matrices an estimator of $U(A)$ subject to a regularizer $\sum_i (\hat{w}^A_i)^2$ that discourages subspaces where the downstream IW estimator would have high variance.  We can estimate $U(A)$ with
\begin{align}
\hat{U}(A)&\coloneqq\hat{L}(b^*;A,\hat{w}^A),\ \text{where}\label{eq:hat_U_A}\\
b^*&=\operatorname{argmin}_{b\in\mathbb{R}^K}\hat{L}(b;A;\hat{w}^A) + c\lVert b\rVert^2,
\end{align} i.e. $b^*$ is a linear predictor with low out-of-sample test loss, obtained via IW loss minimization regularized by a tradeoff constant $c$.  To complete the problem, we still need to construct the projected density ratio estimates $\hat{w}^A_i$, which notably depend on $A$.  We can do so by constraining that the $\hat{w}^A_i$ are equal to the projected density ratio estimates that would be returned by a ratio estimation method, i.e. LSIF or uLSIF, run on the projected covariates $\{A^Tx_i^{tr}\}$, $\{A^Tx_i^{te}\}$.

\subsection{Formulation}

Here, we describe an optimization problem for finding a projection matrix $A\in \{\mathbb{R}^{K\times D}, A^TA=I\}$ with which to project the covariates prior to the application of an IW method for tackling the covariate shift problem, as described in Section \ref{sec:motivation}.  $K<D$ is the dimension of the desired subspace.  In particular, we assume given $A$, a \emph{linear} model will be fit using importance weighting, i.e. minimizing over $b\in\mathcal{R}^D$ the estimator of Equation \ref{eq:estimator} with regularization, and that uLSIF will be used for density ratio estimation.
%, which must be specified. %, where $D<K$ is the dimension of the desired subspace.  
%The intent is that by first projecting the covariates in the covariate shift problem by $A$, subsequent application of an IW method with model class $\mathcal{G}$, i.e. regularized minimization of $\hat{L}(g;A)$ would give a predictor with test loss that is low with high probability.

We propose finding $A$ by solving:
\setlength{\jot}{0pt}
\begin{align}
%&\min_{A}\sum_i w_i^{tr}L(b^{*T}u_i^{tr},y_i^{tr}) + \lambda\sum_i (w_i^{tr})^{2} \label{eq:empirical_obj}\\
&\min_{A}\sum_i \hat{L}(b^*;A,\hat{w}^A) + \lambda\sum_i (\hat{w}_i^{A})^{2} \label{eq:empirical_obj}\\
\shortintertext{subject to}
%&u_i^{tr}=A^Tx_i^{tr},\ u_i^{te}=A^Tx_i^{te}\label{eq:rename}\\
%b^* &= \operatorname{argmin}_{b}\hspace{-0pt}\sum_i \hat{w}(u_i^{tr})L(b^{ T}u_i^{tr},y) + c\lVert b^{}\rVert ^2\label{eq:best_b}\\
%&b^*=\operatorname{argmin}_{b\in\mathbb{R}^K} w_i^{tr}L(b^{*T}u_i^{tr},y_i^{tr}) + c\lVert b\rVert^2\label{eq:is_best}\\
&b^*=\operatorname{argmin}_{b\in\mathbb{R}^K} \hat{L}(b;A,\hat{w}^A) + c\lVert b\rVert^2\label{eq:is_best}\\
&\hat{w}_i^{A} = \max(\bs{\alpha}^{*T}\bs{\phi}(u_i^{tr}),0)\label{eq:pos_ratio}\\
&\bs{\alpha^*}=\operatorname{argmin}_{\boldsymbol{\alpha}} \tfrac{1}{2} \boldsymbol{\alpha}^T\bs{H} \boldsymbol{\alpha^\prime} - \boldsymbol{h}^T\boldsymbol{\alpha} + \gamma \boldsymbol{1}^T\boldsymbol{\alpha}\label{eq:ratio_opt}\\
&\boldsymbol{H}=\tfrac{1}{N^{tr}}\sum_k \boldsymbol{\phi}(A^Tx_i^{tr}) \boldsymbol{\phi}(A^Tx_i^{tr})^T,\\  &\boldsymbol{h}=\tfrac{1}{N^{te}}\sum_k\boldsymbol{\phi}(A^Tx_i^{te})\label{eq:ratio_h}\\
&A^TA=I,\label{eq:stiefel}
\end{align}
where $\hat{L}(b^*;A,\hat{w}^A)$ is as defined in Equation \ref{eq:estimator} and $\lambda,\gamma,c$ are hyperparameters described shortly.

The first term of the objective is the predictive utility estimate for $A$ from Section \ref{sec:motivation}, $\hat{U}(A)$. It has such an interpretation firstly due to the constraint of Equation \ref{eq:is_best}, and secondly because by Equations \ref{eq:pos_ratio}-\ref{eq:ratio_h}, the $\hat{w}_i^{A}$ are the estimates of $\tfrac{P^{te}_{A^TX}(x_i^{te})}{P^{tr}_{A^TX}(x_i^{te})}$ returned by running uLSIF on $\{A^Tx_i^{tr}\},\{A^Tx_i^{te}\}$, the training and test covariates, projected by $A$.  Again, note that the $\hat{w}_i^{A}$ thus depend on $A$ through an optimization problem: that of running uLSIF.  We use uLSIF instead of LSIF for ratio estimation due to its computational benefits and relatively strong performance.  Due to the interpretation of the $\hat{w}_i^{A}$, the second term of the objective is a regularizer that encourages the effective sample size of that estimator, $N^{tr}/\sum_i (\hat{w}_i^A)^2$, to be high.  %Equation \ref{eq:stiefel} ensures that $A$ is a projection matrix. %Figure \ref{fig:diagram} shows a diagram of the formulation. % Finally, $c\lVert b \rVert ^2$ is $L_2$ regularization of the composite linear model $f(x)=b^TA^Tx$, whose coefficient's squared norm is simply $\lVert b \rVert^2$ due to $A$ being orthonormal.  
Regarding hyperparameter selection, we update $c$ and $\gamma$ throughout individual runs of the gradient descent procedure we use with ``in-line'' cross-validation, and we choose $\lambda$, controlling the lower bound on effective sample size, using weighted cross-validation \cite{sugiyama2007covariate} (details in supplement).  
%$\lambda$ controls the tradeoff between the predictive utility of the subspace and the effective sample size within it, $c$ controls the regularization used to evaluate that predictive utility, and $\gamma$ is regularization for uLSIF.

%\begin{figure}
%\centering
%\includegraphics[width=0.35\linewidth]{network_diag.png}
%\caption{The architecture of the formulation}
%\label{fig:diagram}
%\end{figure}

\subsection{Analysis of Dimension Reduction for Importance Weighting\label{sec:analysis}}
We first show that in expectation, following dimension reduction by any projection matrix $A\in \mathbb{R}^{K\times D}$, $N^{tr}/\sum_i(w^A_i)^2)$, the effective sample size of the downstream IW loss estimator of Equation \ref{eq:estimator}, where the density ratio estimates are replaced by the true density ratios, is larger than the effective sample size of the IW loss estimator had no dimension reduction been performed, $N^{tr}/\sum_i(w(x_i^{tr})^2)$:
\begin{lemma}\label{lemma:1}
For $A\in\mathbb{R}^{K\times D}$ s.t. $A^TA=I$, $E_{P_X^{tr}}[N^{tr}/\sum_i(w^A_i)^2)]\geq E_{P_X^{tr}}[N^{tr}/\sum_i(w(x_i^{tr})^2)]$.\hspace*{-39pt}
\end{lemma}
\vspace{-0pt}
\useshortskip
\setlength{\jot}{0pt}
\begin{proof} Let $C$ be a matrix whose columns span the orthogonal complement to the subspace spanned by the columns of $A$, and let $(U,V)=(A^TX,C^TX)$.  Then
\begin{align}
\useshortskip
%&E_{P_X^{tr}}[\tfrac{1}{N^{tr}}\sum_i(w(x_i^{tr})^2)] = E_{P_X^{tr}}[(\tfrac{P^{te}(X)}{P^{tr}(X)})^2]\\
&E_{P_X^{tr}}[\tfrac{1}{N^{tr}}\sum_i(w(x_i^{tr})^2)] = E_{P_X^{tr}}[(\tfrac{P^{te}(X)}{P^{tr}(X)})^2]\\
&= E_{P_{U,V}^{tr}}[(\tfrac{P^{te}(U)P^{te}(V|U)}{P^{tr}(U)P^{tr}(V|U)})^2]\\
&= E_{P_{U}^{tr}}\left[(\tfrac{P^{te}(U)}{P^{tr}(U)})^2E_{P^{tr}_{V|U}}[(\tfrac{P^{te}(V|U)}{P^{tr}(V|U)})^2]\right]\\%\label{eq:two_squares}\\
&\geq E_{P_{U}^{tr}}[(\tfrac{P^{te}(U)}{P^{tr}(U)})] = E_{P_X^{tr}}[\tfrac{1}{N^{tr}}\sum_i (w^A_i)^2] \label{eq:just_U}
\end{align}
\end{proof}
\vspace{-0pt}
where we have used the fact that for any fixed $U$, $1 \leq E_{P^{tr}_{V|U}}[(\tfrac{P_{V|U}^{te}(V|U)}{P_{V|U}^{tr}(V|U)})^2]\coloneqq PE(P^{te}_{V|U}\lVert P^{te}_{V|U})+1$ to justify Equation \ref{eq:just_U}, which follows because $PE(\cdot\lVert\cdot)$, denoting the Pearson Divergence between distributions, is always at least 0.  This verifies our intuition that projecting onto a subspace can, under idealized circumstances, increase effective sample size.  Though, as the effective sample size following projection is only larger in expectation, the realized effective sample size might still be small for some projections, thus the need to explicitly regularize against such cases.

\newcommand{\train}{}
\newcommand{\test}{}

While dimension reduction reduces downstream estimator variance (in expectation), it introduces bias in the estimator.  Under the covariate shift assumption, the estimator of Equation \ref{eq:estimator}, with true density ratios replacing estimates thereof, does not unbiasedly estimate test loss.
\begin{lemma}\label{lemma:2}
%$E_{P_X^{tr}}[\tfrac{1}{N^{tr}}\sum_i w^A_iL\left(b^TA^Tx_i^{tr},y_i^{tr} \right)]\neq E_{P_{X,Y}^{te}}[L\left(b^TA^TX,Y \right)]$, for $b\in\mathbb{R}^D$, $A$ as in Lemma \ref{lemma:1}. \hspace*{-0pt}
\hspace*{-2.5pt}\mbox{$E_{\train{P^{tr}_{X,Y}}}[\hat{L}(b;A,w^A)]\neq E_{P_{X,Y}^{te}}[L\left(b^TA^TX,Y \right)]$}  for $b\in\mathbb{R}^D$, $A$ as in Lemma \ref{lemma:1}. \hspace*{-0pt}
\end{lemma}
%\hspace*{-0pt}
\begin{proof}Let $U,V$ be as in Lemma \ref{lemma:1}, $(x_i^{tr},y_i^{tr})\sim P^{tr}_{X,Y}$.  Then,
\begin{align*}
%&E_{\train{P^{tr}_{U,Y}}}[\tfrac{1}{N^{tr}}\sum_i w^A_iL\left(b^TA^Tx_i^{tr},y_i^{tr} \right)]\\
&E_{\train{P^{tr}_{X,Y}}}[\hat{L}(b;A,w^A)]\coloneqq E_{\train{P^{tr}_{U,Y}}}[\tfrac{1}{N^{tr}}\sum_i w^A_iL\left(b^TA^Tx_i^{tr},y_i^{tr} \right)]\\
&= E_{\train{P^{tr}_{U,Y}}}[\tfrac{\test{P^{te}_{U}}(U)}{\train{P^{tr}_{U}}(U)}L(f(U),Y)]\\
&=E_{\train{P^{tr}_{U,V}}}[\tfrac{\test{P^{te}_{U}}(U)}{\train{P^{tr}_{U}}(U)}E_{P_{Y|U,V}}[L(f(U),Y)]]\\
&=\int \train{P^{tr}_U}(U) \train{P^{tr}_{V|U}}(U)
   \tfrac{\test{P^{te}_{U}}(U)}{\train{P^{tr}_{U}}(U)}E_{P_{Y|U,V}}
   [L(f(U),Y)]] dU dV \\
&=\int \test{P^{te}_U}(U) \train{P^{tr}_{V|U}}(U)E_{P_{Y|U,V}}
   [L(f(U),Y)]] dU dV \\
&=E_{\test{P^{te}_U}\train{P^{tr}_{V|U}}P_{Y|U,V}}[L(f(U),Y)] \neq E_{P_{X,Y}^{te}}[L\left(b^TA^TX,Y \right)].
\end{align*}
\end{proof}

To understand the source of this bias, we can upper bound the bias with an interpretable quantity (proof in supplement):
\begin{lemma}
Given $A,b$, let $U,V$ be as defined in Lemma \ref{lemma:1}, and $(x_i^{tr},y_i^{tr})\sim P^{tr}_{X,Y}$.  Then
%\begin{flalign}\hspace{-0pt}\Big|E_{P^{tr}_{X,Y}}[\tfrac{1}{N^{tr}}\sum_iw^A_i&L(b^T(A^Tx_i^{tr}),y_i^{tr})] \nonumber\\[-1em]
\begin{align}\hspace{-0pt}&\Big|E_{P^{tr}_{X,Y}}[\hat{L}(b;A,w^A)]
- E_{P^{te}_{X,Y}}[l(b^TA^TX,Y)]\Big| \nonumber\\
&\leq E_{P_U^{tr}}[PE(P^{te}_{V|U}\lVert P^{tr}_{V|U})]^{\tfrac{1}{2}} E_{P^{te}_U}[Var_{P^{tr}_{Y|U}}(l(b^TU,Y))]^{\tfrac{1}{2}}.\nonumber
\end{align} \label{lemma:3}
\end{lemma}

The LHS is the absolute bias of our estimator.  The first term of the bound measures, roughly speaking, how much the training and test distributions differ in the subspace orthogonal to the subspace $A$ parameterizes.  The second measures how much variance there is in $l(b^TU,Y)$, with $U$ fixed, averaged over the test distribution of $U$, i.e. it represents how well the loss can be predicted, given only $U$.  The bias upper bound is high if both terms are simultaneously high.
%In particular, this term is $0$ if $Y\perp X | A^TX$, i.e. $A$ parameterizes a sufficient subspace of the original covariate space.  The bias is high only if both terms are high.  

This analysis studies the bias and variance of the downstream IW estimator $\hat{L}(b;A,\hat{w}^A)$, under the assumption that $A$ is fixed.  But of course $A$ is not fixed, and furthermore, its selection uses that same IW estimator when obtaining a subspace's estimated predictive utility $\hat{U}(A)$ of Equation \ref{eq:hat_U_A}.  Therefore, the bias and variance in $\hat{L}(b;A,w^A)$ also affects the selection of $A$.  In particular,  the total variance of our procedure contains the variability in selecting $A$.  This is why in experiments, the loss variance of our two-step procedure can be sometimes higher than IW, even if conditional on $A$ being chosen, we would expect the downstream IW method to have lower variance.  Also, of keen interest would be a \emph{uniform} bound on bias, over \emph{all} $A$ and $b$.
%(c.f. Equation \ref{eq:empirical_obj} and \ref{eq:is_best}).

%Thus far, this analysis has analyzed the effect of projection by $A$ on the downstream IW estimator $\hat{L}(b;A,w^A)$ of Equation \ref{eq:estimator} (and thus minimization procedure), with $A$ \emph{fixed}.  However, $A$ is of course not fixed, and furthermore, $\hat{L}(b;A,w^A)$ is also used in selecting $A$ (c.f. Equations \ref{eq:empirical_obj} and \ref{eq:is_best}).  This has implications.  Firstly, due to the need to accurately choose $A$, we would like $\hat{U}(A)\approx U(A)$ for all $A$ we consider, which we might accomplish through suitable selection of hyperparameters $\lambda$ and $D$.  To a priori evaluate whether this is so for a given $\lambda$ and $D$, we would want to know the resulting \emph{uniform} upper bounds on both bias and variance of $\hat{L}(b;A,w^A)$ over \emph{all} $A$ and $b$ that we consider.  While we can control that uniform upper bound on the variance via the $\lambda$ hyperparameter, the bound of Lemma \ref{lemma:3} does not inform what that uniform bound on the bias would be.  Therefore, choosing $\lambda$ and $D$ is best done using cross-validation.  Secondly, although we can choose an $A$ such that the downstream IW method has low variance, the total variance of our procedure contains the variability in selecting $A$.  This is why in experiments, the loss variance of our two-step procedure can be sometimes higher than IW.

\subsection{Solving the Optimization Problem}
The objective of the optimization problem, which we will refer to as $G$, depends only on projection matrix $A$, through several intermediate variables.  Thus, 
we solve the optimization problem via gradient descent over projection matrices using the Pymanopt \cite{JMLR:v17:16-177} package.  The challenge is in calculating the gradient $\tfrac{dG}{dA}$, as intermediate variables $\bs{\alpha^*}$ and $\bs{b^*}$ depend on other variables not analytically, but as the solution to convex optimization problems parameterized by the dependent variables.  We call them $\operatorname{argmin}$ variables. Fortunately we can apply existing work \cite{foo2008efficient,maclaurin2015gradient} to efficiently calculate $\tfrac{dG}{dA}$. We used Autograd\cite{maclaurin2015autograd} to calculate gradients not involving $\operatorname{argmin}$ variables.  As the problem is not convex in $A$, we utilize multiple random restarts.

\textbf{Differentiation with $\operatorname{argmin}$ variables}:
In reverse mode differentiation, the generic task is to recursively, given an objective of the form $G(v(u))$ and $\tfrac{dG}{dv}$, to compute $\tfrac{dv}{du}\tfrac{dG}{dv}$.  In our problem, we use reverse mode differentiation to compute $\tfrac{db^*}{dw}\tfrac{dG}{db^*}$ given $\tfrac{dG}{db^*}$, where $w$ is the length $N^{tr}$ vector of weights $\{\hat{w}_i^A\}$ of Equation \ref{eq:pos_ratio}, and $\tfrac{d\alpha^*}{dA}\tfrac{dG}{d\alpha^*}$ given $\tfrac{dG}{d\alpha^*}$.  For concreteness, we will illustrate the latter calculation; the same technique suffices for both, as both $\operatorname{argmin}$ variables are the solutions to unconstrained problems.  Please see \cite{amos2017optnet} for how to calculate $\tfrac{d\alpha^*}{dA}\tfrac{dG}{d\alpha^*}$ if LSIF is used, where $\alpha^*$ is given by a constrained optimization problem. 

\textbf{Calculating $\tfrac{db^*}{dw}$}:
We first describe how to explicitly form $\tfrac{db^*}{dw}$, as the naive way to calculate $\tfrac{db^*}{dw}\tfrac{dG}{db^*}$ is simply to form $\tfrac{db^*}{dw}$ and then matrix multiply.  Since $b^*$ minimizes the convex function $f(b)\coloneqq \hat{L}(b^*;A,\hat{w}^A)+ c \lVert b \rVert^2$, $b^*$ satisfies the stationarity condition
%\begin{align}
$\tfrac{df}{db}(b^*(w),w) = 0,$
where $0$ is a length D vector of zeros, and the notation suggests $f$ depends on both $b^*$ and $w$, and $b^*$ further depends on $w$.  Differentiating with respect to $w$ gives
%\tfrac{db^*}{d\hat{w}}(\hat{w}) \tfrac{d}{db}\tfrac{df}{db}(b^*(\hat{w}),\hat{w})
$\tfrac{db^*}{dw}(\tfrac{d}{db}\tfrac{df}{db}) + \tfrac{d}{dw}\tfrac{df}{db} = 0,$
where $\tfrac{db^*}{dw}$ is the desired $N\times K$ Jacobian matrix, $\tfrac{d}{db}\tfrac{df}{db}$ is the $K\times K$ Hessian matrix of $f$, and $\tfrac{d}{dw}\tfrac{df}{db}$ is $N \times K$.  Rearranging, we obtain a multiple linear system we can solve for $\tfrac{db^*}{dw}$:
$\tfrac{d}{db}(\tfrac{df}{db}) \tfrac{db^*}{dw}^T = -(\tfrac{d}{dw}\tfrac{df}{db})^T.$
%\end{align}
However solving this multiple linear system is in general not feasible, as calculating $\tfrac{db^*}{dw}$ involves solving $N$ separate linear systems, with $w$ being $N^{tr}$-dimensional. % , and $D\times K$ in the case of calculating the also required $\tfrac{d\bs{\alpha}^*}{d\hat{A}}$

\textbf{Efficient Calculation of $\tfrac{db^*}{dw}\tfrac{dG}{db^*}$}:
Fortunately, two computational tricks can be applied.  Firstly, we can solve a single linear system instead of $N$ of them.  We can express the desired gradient as\vspace{-0pt}
%\begin{align}
$\tfrac{db^*}{dw}\tfrac{dG}{db^*} = -\tfrac{d}{dw}\tfrac{df}{db} (\tfrac{d}{db}\tfrac{df}{db})^{-1}\tfrac{dG}{db^*}.$
As $\tfrac{dG}{db^*}$ is assumed available, we can first solve
$(\tfrac{d}{db}\tfrac{df}{db}) v = \tfrac{dG}{db^*} \label{eq:v}$
%\end{align}
for $v$, and then left multiply by $-\tfrac{d}{dw}\tfrac{df}{db}$ to get $\tfrac{db^*}{dw}\tfrac{dG}{db^*}$.

Moreover, we can avoid explicitly constructing the Hessian matrix in this linear system.  %., which in the case of calculating $\tfrac{dw}{dA}\tfrac{dG}{dw}$, is $\tfrac{d}{d\alpha}\tfrac{dg}{d\alpha}$, of large size. 
Matrix free linear solvers such as conjugate gradient, when solving a system $Cx=d$, do not explicitly require $C$, but only that the matrix vector product $Cu$ be able to be calculated for any vector $u$.  Furthermore, Hessian-vector products can be calculated efficiently.  Returning to the example, note that for any $u$,
%\begin{align}
$(\tfrac{d}{db}\tfrac{df}{db}) u = \tfrac{d}{db}(u^T\tfrac{df}{db}).$
%\end{align}
This means the requisite Hessian vector can be computed by first analytically forming the scalar valued function $u^T\tfrac{df}{db}$ and then calculating its gradient with respect to $b$ analytically or via finite differences.

\section{Simulation Study\label{sec:sim}}
We now apply extreme dimension reduction to two synthetic examples in order to give insight into when the bias in the IW estimator (c.f. Lemma \ref{lemma:2}) is detrimental.  In the first, this bias is nonexistent.  In the second, this bias is crippling.

\subsection{Example with no estimator bias}
We generated 12-dimensional covariates and real valued labels.  The labels only depend on the first two covariates, $X_1$ and $X_2$.  In particular, we let $Y|X \sim N(0.2\lVert X_1 \rVert +  \lVert X_2 \rVert, 0.01)$.
Figure \ref{fig:sim_data_dim1} shows samples from $P^{tr}_{Y,X_1}$ and $P^{te}_{Y,X_1}$, and Figure \ref{fig:sim_data_dim2} shows samples from $P^{tr}_{Y,X_2}$ and $P^{te}_{Y,X_2}$.  Training samples are in red, test samples are in test.  Thus for $X_1$ and $X_2$, the test distribution contains only positive values.  The distributions of each covariate are independent of each other, generated as follows:
\begin{align*}
&X_1^{tr},X_2^{tr}\sim \operatorname{Uniform}(-1,1)\\
&X_1^{te},X_2^{te}\sim \operatorname{Uniform}(0,1)\\
%&X_2^{tr}\sim \operatorname{Uniform}(-1,1),\ X_2^{te}\sim \operatorname{Uniform}(0,1)\\
&X_3^{tr},\dots,X_{12}^{tr}\sim 0.9 \operatorname{Uniform}(-1,0) + 0.1 \operatorname{Uniform}(0,1)\\ &X_3^{te},\dots,X_{12}^{te}\sim 0.1 \operatorname{Uniform}(-1,0) + 0.9 \operatorname{Uniform}(0,1)
%&X_{13}^{tr}, X_{13}^{te}\sim \operatorname{Uniform}(-1,1)
\end{align*}

\begin{table}[]
\centering
%\footnotesize
\begin{tabular}{c|c|c|c|c}
& \multicolumn{3}{c}{Data size} \\
\hline
Method & 50 & 100 & 150 & 200\\
\hline \hline
JP(1) & .16(.08) & .13(.04)  &.13(.03) & .14(.03)\\
UW & .26(.04) & .26(.03) &.26(.02) &.25(.02)\\
IW & .26(.04) & .26(.02) &.25(.02) & .25(.02)\\
SIR$(1)$ & .26(.03) & .26(.03) &.26(.03) & .26(.02)\\
RP$(1)$ & .26(.04) & .26(.03) &.25(.02) & .25(.02)\\
%CT & .59(.39) & .83(.01) & .45(.23) & .30(.11)
\end{tabular}
\caption{Example 1 - test loss over 50 replicates.
\label{fig:sim_data_loss}}
\vspace{-0pt}
\end{table}

\vspace{-0pt}Note that as our model class contains only linear models, when trained and tested on $P^{tr}$, a model that only uses $X_1$ will have comparable performance compared to a model that only uses $X_2$.  This is because even though $X_2$ appears more informative, the model class does not contain the ``v-shaped'' predictor that the mean of $P_{Y|X_2}$ follows.  Contrarily, when trained and tested on $P^{te}$, a model that uses only $X_2$ will have far superior performance to one that uses only $X_1$, because due to the covariate shift, the test covariates only have support under one of the two ``arms'' of the ``v''.  

\begin{figure}[]
\centering
\begin{subfigure}{0.48\linewidth}
%\centering 
%\includegraphics[trim = 40 200 40 200, width=1\linewidth]{\getFullPath{age_scatters.pdf}}
\includegraphics[width=1\linewidth]{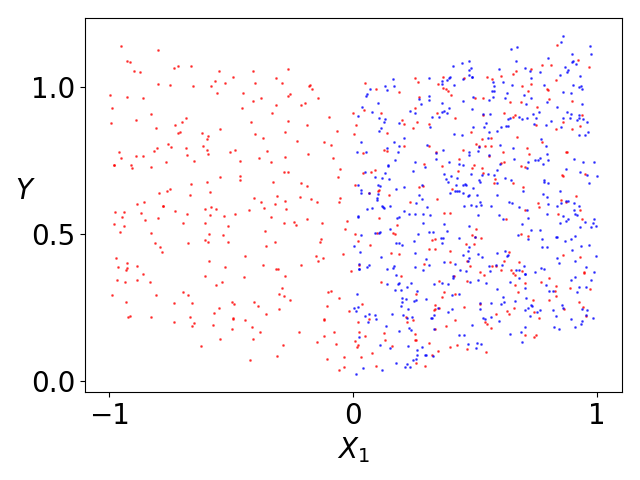}
\caption{$X_1$ is moderately predictive in both the test and training distribution.\\}
\label{fig:sim_data_dim1}
\end{subfigure}%~
\begin{subfigure}{0.48\linewidth}
%\centering
\includegraphics[width=1\linewidth]{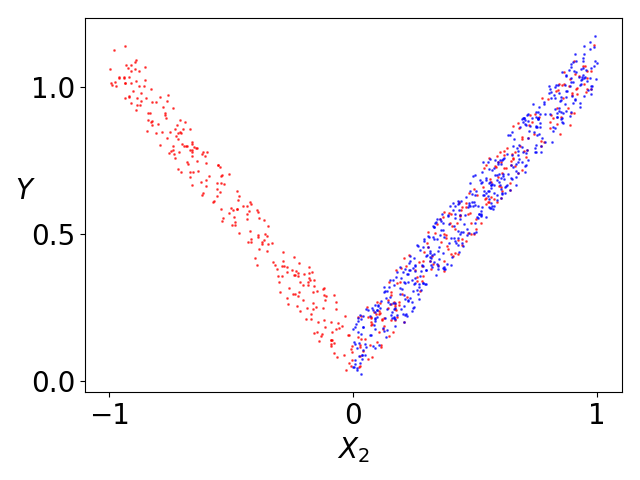}
\caption{$X_2$ is very predictive in the test distribution, but not very predictive in the training distribution.}
\label{fig:sim_data_dim2}
\end{subfigure}
\caption{Example 1 - the relative importance between the only 2 predictive features is reversed between training \& test distributions.  Training points in red, test in blue.}
\label{fig:sim_data}
\end{figure}

We compare our method (dimension reduction, then applying uLSIF), denoted JP$(K)$, K the subspace dimension, to:% the following:
\begin{itemize}[topsep=0pt]
\item Unweighted (UW): performing no covariate shift correction, minimizing (regularized) unweighted loss over the training data
\item Naive Importance Weighting (IW): Not applying any dimension reduction, and then applying uLSIF.
\item Random Projection to $K$ dimensions (RP$(K)$): Applying a random projection to generate $K$-dimensional covariates,  then applying uLSIF to the projected data.
\item Sliced Inverse Regression (SIR$(K)$): Applying sliced inverse regression \cite{li1991sliced} to project to $K$-dimensions, and then applying uLSIF.
%\item Cheating (CT): Projecting the data to only retain the first two covariates, the only ones upon which the outcome depends, and then applying uLSIF.
\end{itemize}

Table \ref{fig:sim_data_loss} shows the mean and standard deviation of out-of-sample test loss (measured by absolute prediction error) over 50 replicates for the methods, as $N$, the number of generated training and test samples changes.  To elaborate, we generate $N$ samples each from $P^{tr}_{X,Y}$ and $P_X^{te}$, set aside $\tfrac{1}{3}$ of the test data, fit a model using the $N$ labelled training samples and remaining $\tfrac{2N}{3}$ unlabelled test samples, and evaluate the predictions on the set aside test data.  %Losses are normalized so UW has a loss of 1 for each data size.  

We note the performance of IW is often worse than that of UW, due to the difficulty in estimating the weights and small sample size.  Furthermore, the sliced inverse regression methods, which do not account for covariate shift in estimating the projection, do not find any subspace particularly informative and thus do about the same as the random projection methods.  On the other hand, our method is able to find that $X_2$ is most useful in the test domain, so that the downstream IW method is applied to the correct 1-dimensional subspace.  We note that the variance in the loss for our method is slightly increased, due to the variance in the selection of the subspace within which to apply IW.
%For some data sizes, our method even does better than CT, which is told the ground truth that only $X_1$ and $X_2$ are relevant.  Even though $Y$ does depend on both features, the combination of weights being harder to estimate and effective sample size being smaller in 2 dimensions versus 1 shows the benefit of extreme dimension reduction to very low dimensions.

Our dimension reduction method projects the covariates to retain only $X_2$, because it finds $\hat{U}(X_2)\gg\hat{U}(X_1)$ (by abuse of notation we will use $X_2$ also to refer to the 1-dimensional projection matrix that retains only $X_2$).  Why is it able to accurately estimate $\hat{U}(X_1)$ and $\hat{U}(X_2)$ despite the fact that their estimation uses estimators of a predictor $b$'s test loss, $\hat{L}(b;X_1,\hat{w}^{X_1})$ and $\hat{L}(b;X_2,\hat{w}^{X_2})$, that are potentially biased, by Lemma \ref{lemma:2}?  The reason is that we have designed the data distributions so that $\hat{L}(b;X_1,w^{X_1})$ and $\hat{L}(b;X_2,w^{X_2})$, the same estimators but with true density ratios replacing their estimates, are actually unbiased.  We show this in the case of $X_1$. Since $\hat{L}(b;X_1,w^{X_1})$ is a function of just $X_1,X_2,Y$, in calculating its expectation with respect to any distribution, we can marginalize over variables other than $X_1,X_2,Y$.  Specifically, in calculating $E_{P^{tr}}[\hat{L}(b;X_1,w^{X_1})]$, we can invoke Lemma \ref{lemma:2} with $U=X_1,V=X_2$, to obtain \begin{align}
E_{P^{tr}}[\hat{L}(b;X_1,w^{X_1})] = E_{P^{1\setminus 2}}[\hat{L}(b;X_1,w^{X_1})],\label{eq:1_2_1}\hspace{-5pt}\\
\text{where}\ P^{1\setminus 2}_{X_1,X_2,Y}\coloneqq\test{P^{te}_{X_1}}\train{P^{tr}_{X_2|X_1}}P_{Y|X_1,X_2}.\label{eq:1_2}
\end{align}
However, actually $\train{P^{tr}_{X_2|X_1}}P_{Y|X_1,X_2}=P^{te}_{Y,X_2|X_1}$ (due to the symmetry in $\lVert X_2 \rVert$).  This implies $P^{1\setminus 2}_{X_1,X_2,Y} = P^{te}_{X_1,X_2,Y}$ by Equation \ref{eq:1_2}.  Specifically, expectations with respect to these 2 distributions are equal, so that by Equation \ref{eq:1_2_1}, $E_{P^{tr}}[\hat{L}(b;X_1,w^{X_1})]=E_{P^{te}}[\hat{L}(b;X_1,w^{X_1})]$, i.e. $\hat{L}(b;X_1,w^{X_1})$ unbiasedly estimates the test loss of $b$.

\subsection{Example with crippling estimator bias}
Now, we slightly modify the data distribution from the previous example, and obtain an example where the previously unbiased estimators suffer from crippling bias.  In particular, we make a single change, letting
\begin{align}
X_2^{tr}|X_1^{tr} = \begin{cases}X_1^{tr}\ \text{wpr}\ 0.5\\-X_1^{tr}\ \text{otherwise}.\end{cases}
\end{align}

\begin{figure}[]
\centering
\begin{subfigure}{0.48\linewidth}
%\centering 
%\includegraphics[trim = 40 200 40 200, width=1\linewidth]{\getFullPath{age_scatters.pdf}}
\includegraphics[width=1\linewidth]{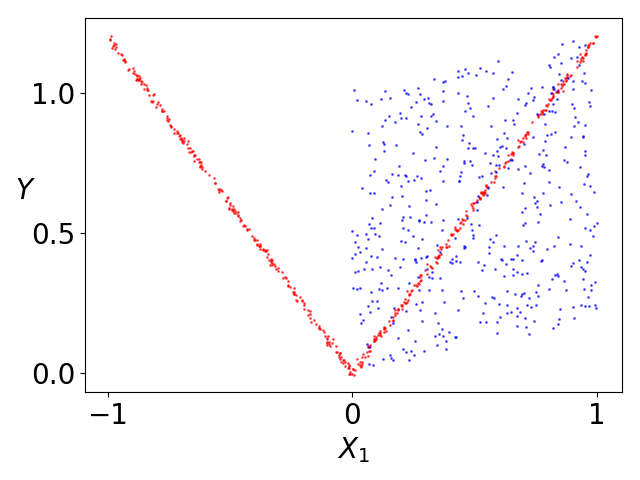}
\caption{$P^{tr}_{X_1,Y}$ and $P^{te}_{X_1,Y}$}%$X_1$ is moderately predictive in both the test and training distribution.\\}
\label{fig:sim_data_dim1_2}
\end{subfigure}%~
\begin{subfigure}{0.48\linewidth}
%\centering
\includegraphics[width=1\linewidth]{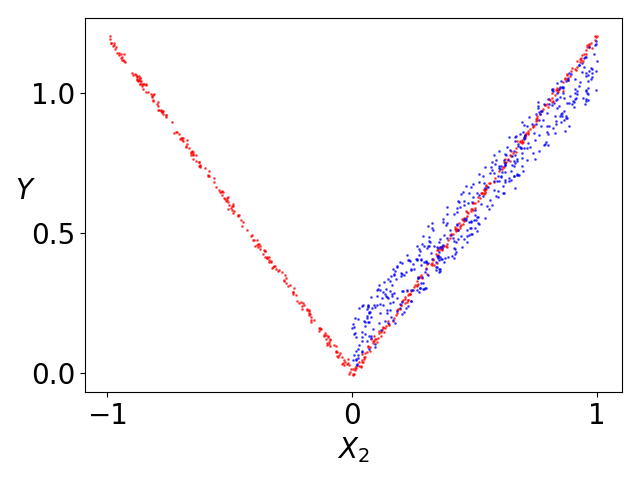}
\caption{$P^{tr}_{X_2,Y}$ and $P^{te}_{X_2,Y}$}%$X_2$ is very predictive in the test distribution, but not very predictive in the training distribution.}
\label{fig:sim_data_dim2_2}
\end{subfigure}
\caption{Example 2 - $P^{te}$ is unchanged from before, so that $X_2$ is still most predictive in the test distribution.  Training points in red, test in blue.}
\label{fig:sim_data_2}
\end{figure}
For this second example, Figure \ref{fig:sim_data_dim1_2} shows samples from $P^{tr}_{Y,X_1}$ and $P^{te}_{Y,X_1}$, and Figure \ref{fig:sim_data_dim2_2} shows samples from $P^{tr}_{Y,X_2}$ and $P^{te}_{Y,X_2}$.  $P^{te}$ is unchanged, and $X_1$ and $X_2$ remain not predictive in $P^{tr}$.  However, the important difference in this example is that in $P^{1\setminus 2}_{X_1,X_2,Y}$ and $P^{2\setminus 1}_{X_1,X_2,Y}$, where by the same reasoning of Equations \ref{eq:1_2_1} and \ref{eq:1_2},
\begin{align}
E_{P^{tr}}[\hat{L}(b;X_2,w^{X_2})] = E_{P^{2\setminus 1}}[\hat{L}(b;X_2,w^{X_2})],\label{eq:2_2_1}\hspace{-5pt}\\
\text{where}\ P^{2\setminus 1}_{X_1,X_2,Y}\coloneqq\test{P^{te}_{X_2}}\train{P^{tr}_{X_1|X_2}}P_{Y|X_1,X_2}.\label{eq:2_2}
\end{align}

\begin{figure}[]
\centering
\begin{subfigure}{0.48\linewidth}
%\centering 
%\includegraphics[trim = 40 200 40 200, width=1\linewidth]{\getFullPath{age_scatters.pdf}}
\includegraphics[width=1\linewidth]{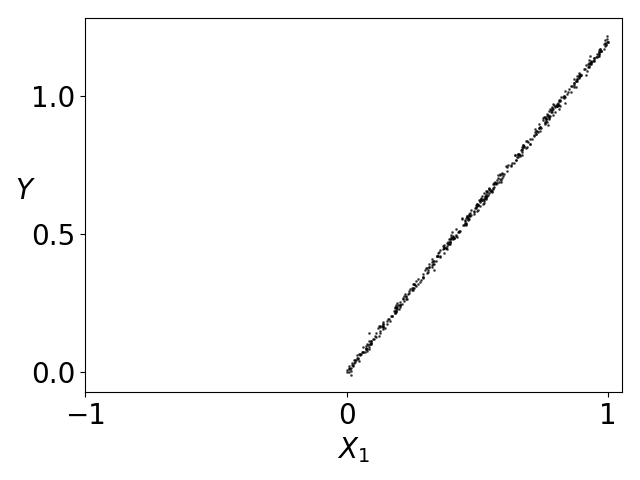}
\caption{$P^{1\setminus 2}_{X_1,Y}$}
\label{fig:sim_data_dim1_whack}
\end{subfigure}%~
\begin{subfigure}{0.48\linewidth}
%\centering
\includegraphics[width=1\linewidth]{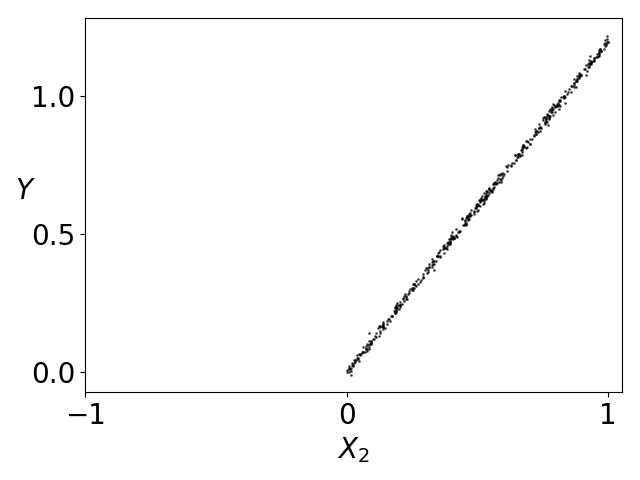}
\caption{$P^{2\setminus 1}_{X_2,Y}$}
\label{fig:sim_data_dim2_whack}
\end{subfigure}
\caption{Example 2 - Our method erroneously finds $X_1$ and $X_2$ to be equally predictive in the test distribution.}
\label{fig:sim_data_whack}
\end{figure}

Therefore, to understand the bias in $\hat{L}(b;X_1,w^{X_1})$ and $\hat{L}(b;X_2,w^{X_2})$ for any $b$, we only need to examine $P _{X_1,Y}^{1\setminus 2}$ and $P _{X_2,Y}^{2\setminus 1}$ (note the estimators are functions of $X_1,Y$ and $X_2,Y$, respectively).  Figure \ref{fig:sim_data_dim1_whack} shows samples from $P _{X_1,Y}^{1\setminus 2}$, and Figure \ref{fig:sim_data_dim2_whack} shows samples from $P_{X_2,Y}^{2\setminus 1}$.  What we see is that those two distributions are \emph{identical}.  This means that our method will find $\hat{U}(X_1)\approx \hat{U}(X_2)$, so that the 1-dimensional projection  will contain equal contributions from $X_1$ and $X_2$.  As $P^{te}$ is unchanged between the two examples, it remains the fact that the ideal projection to apply before learning a model for $P^{te}$ would retain only $X_2$.  Any other projection would be suboptimal.  This is reflected in Table \ref{fig:sim_data_loss_2}, where the performance of our method is degraded, so that they approach that of the other baseline methods in the first example.  We can verify the reason for this degradation by examining the component of the 1-dimensional projection vector corresponding to $X_1$ and $X_2$, for both examples.  We do so when $N^{tr}=N^{te}=200$, i.e. fairly large.  The average value of those 2 components for the 1st example are $0.01,0.99$, respectively.  The corresponding values for the 2nd example are $0.49,0.53$, respectively, reflecting that in the 2nd example, our method finds $X_1$ and $X_2$ equally predictive.

\begin{table}[]
\centering
\begin{tabular}{c|c|c|c|c}
& \multicolumn{3}{c}{Data size} \\
\hline
Dataset & 50 & 100 & 150 & 200\\
\hline \hline
Example 1 & .16(.08) & .13(.04)  &.13(.03) & .14(.03)\\
Example 2 & .26(.10) & .20(.11) &.23(.11) &.22(.13)\\
\end{tabular}
\caption{Comparison of our method's performance on the 2 simulated examples.  
\label{fig:sim_data_loss_2}}
\vspace{-0pt}
\end{table}

\vspace*{-5pt}

\begin{figure*}
\begin{subfigure}[b]{.16\linewidth}
\includegraphics[width=\linewidth]{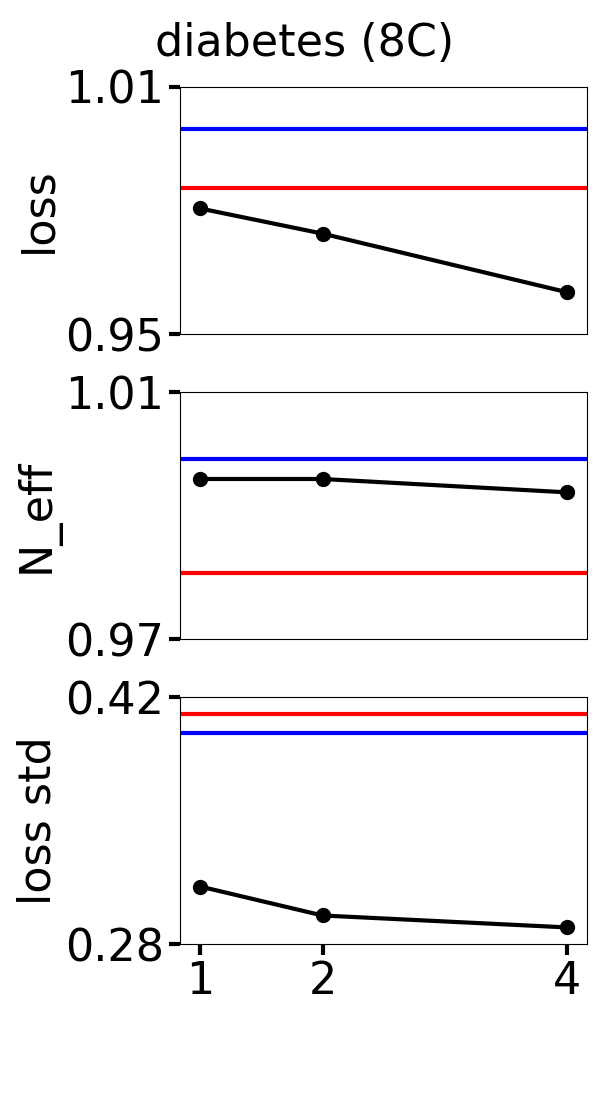}
%\caption{A gull}\label{fig:gull}
\end{subfigure}
\begin{subfigure}[b]{.16\linewidth}
\includegraphics[width=\linewidth]{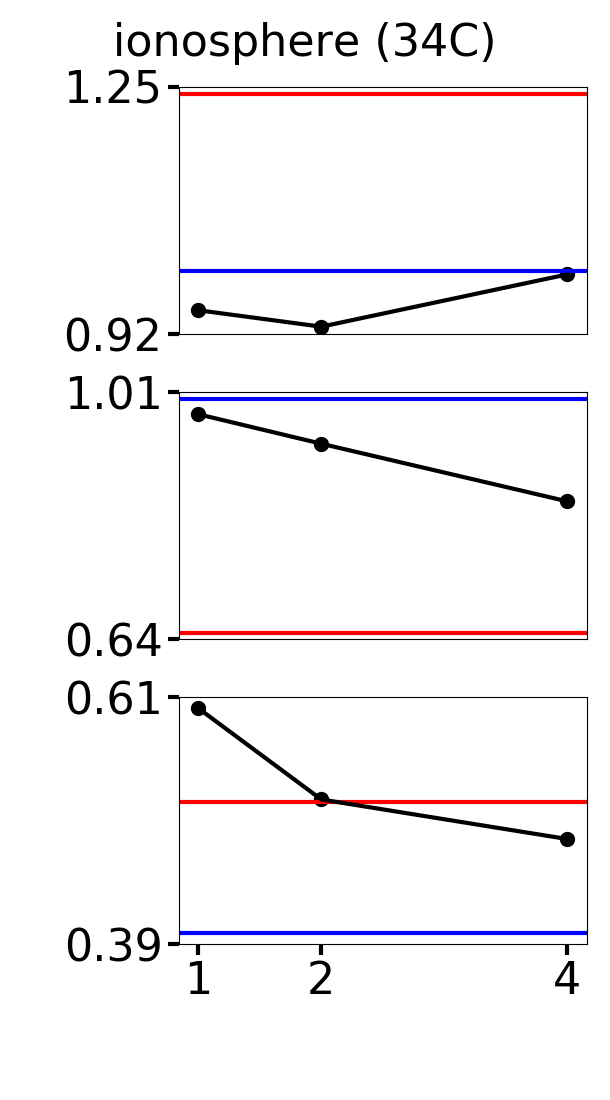}
%\caption{A tiger}\label{fig:tiger}
\end{subfigure}
\begin{subfigure}[b]{.16\linewidth}
\includegraphics[width=\linewidth]{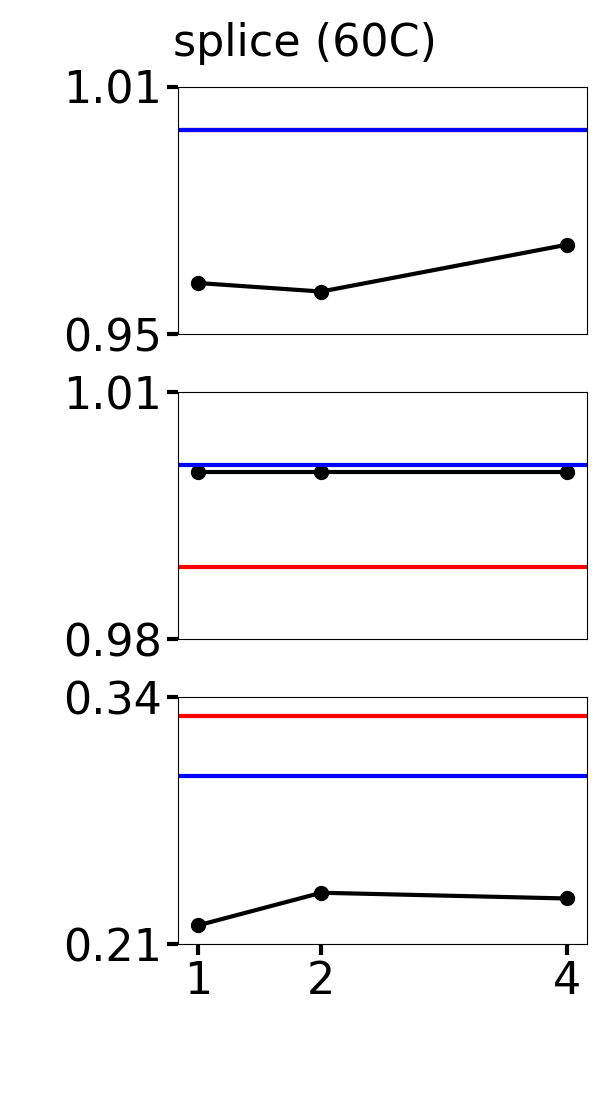}
%\caption{A gull}\label{fig:gull}
\end{subfigure}
\begin{subfigure}[b]{.16\linewidth}
\includegraphics[width=\linewidth]{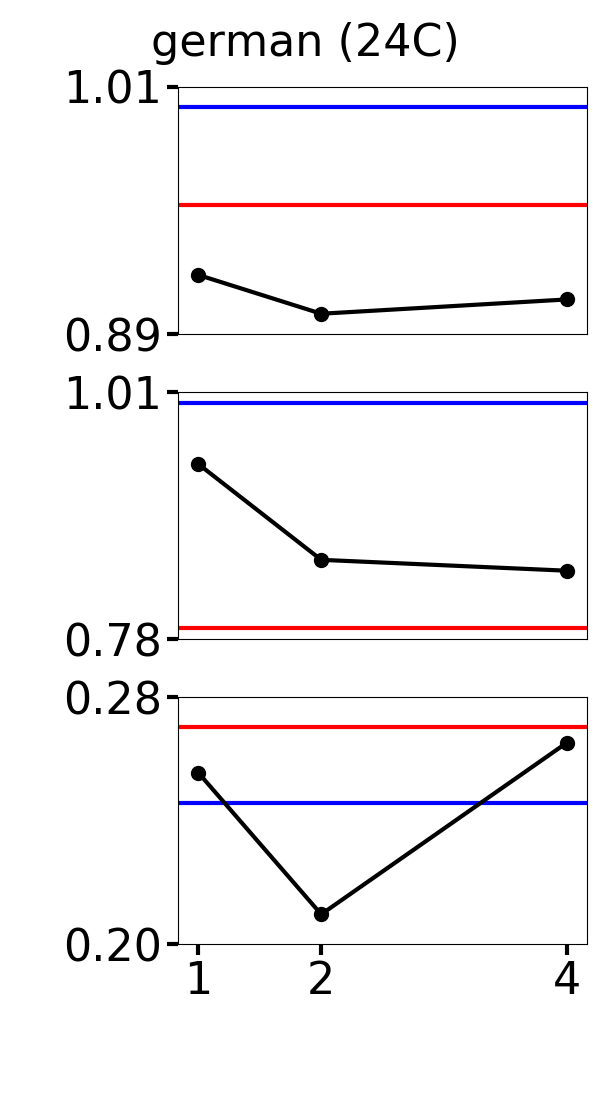}
%\caption{A tiger}\label{fig:tiger}
\end{subfigure}
\begin{subfigure}[b]{.16\linewidth}
\includegraphics[width=\linewidth]{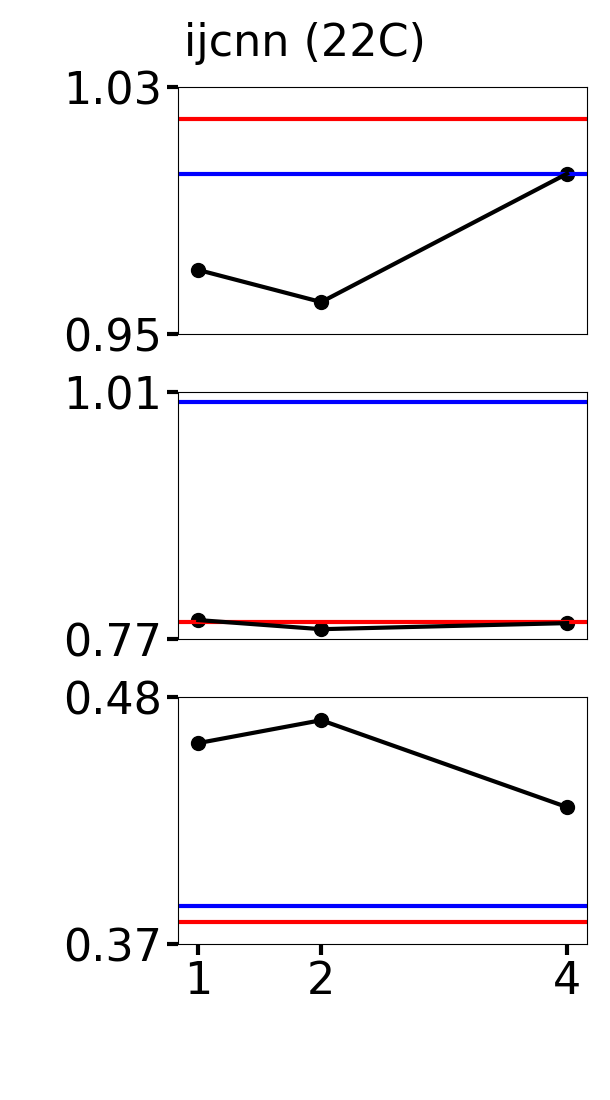}
%\caption{A gull}\label{fig:gull}
\end{subfigure}
\begin{subfigure}[b]{.16\linewidth}
\includegraphics[width=\linewidth]{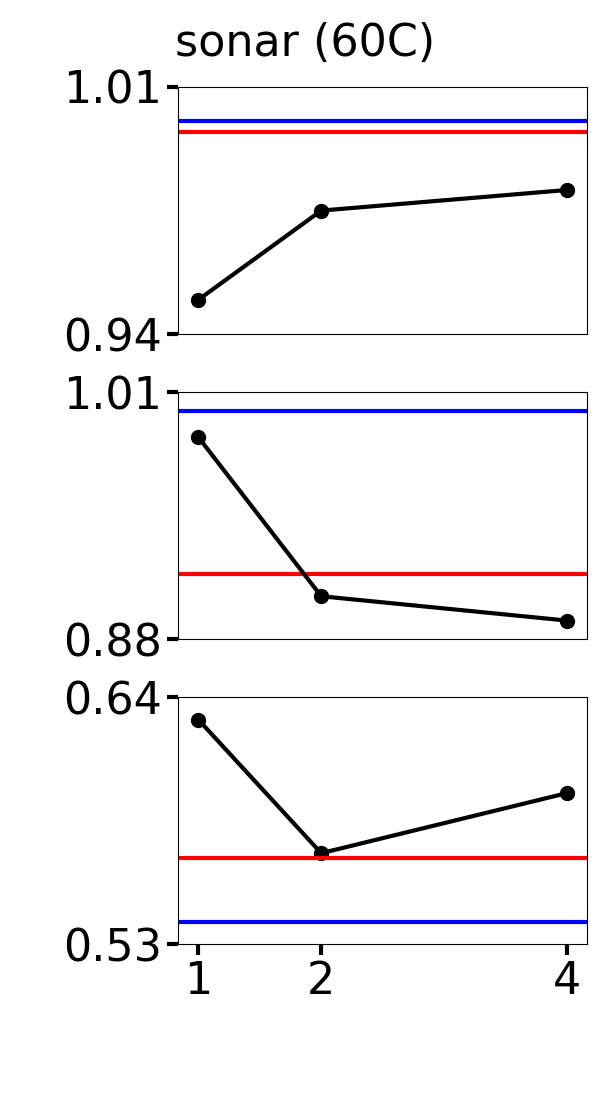}
%\caption{A tiger}\label{fig:tiger}
\end{subfigure}
%\caption{Picture of animals}
\label{fig:animals}\vspace*{0pt}\\
\begin{subfigure}[b]{.16\linewidth}
\includegraphics[width=\linewidth]{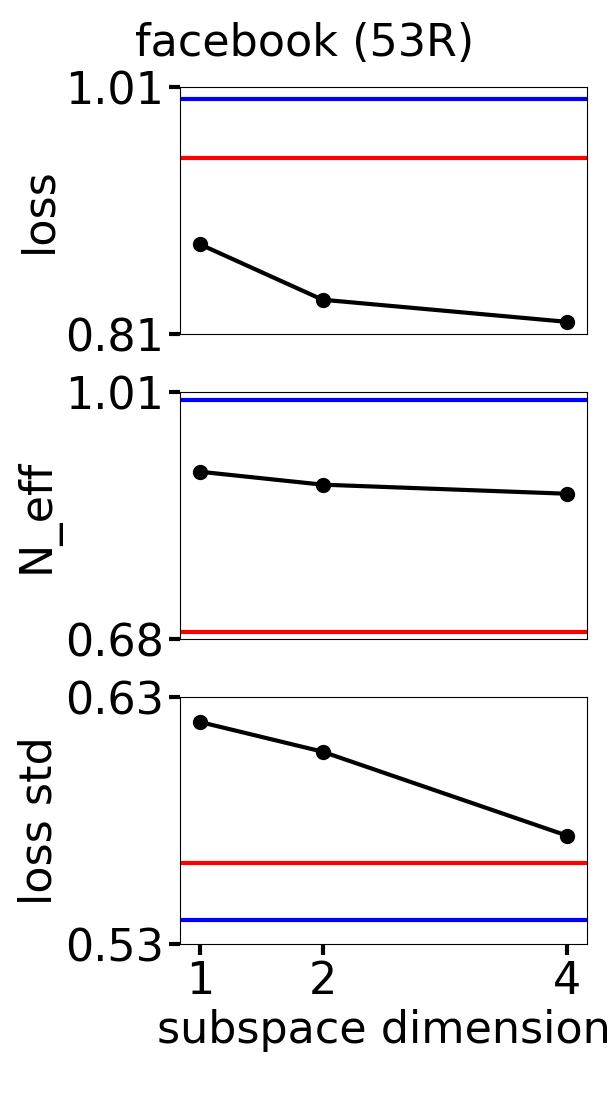}
%\caption{A gull}\label{fig:gull}
\end{subfigure}
\begin{subfigure}[b]{.16\linewidth}
\includegraphics[width=\linewidth]{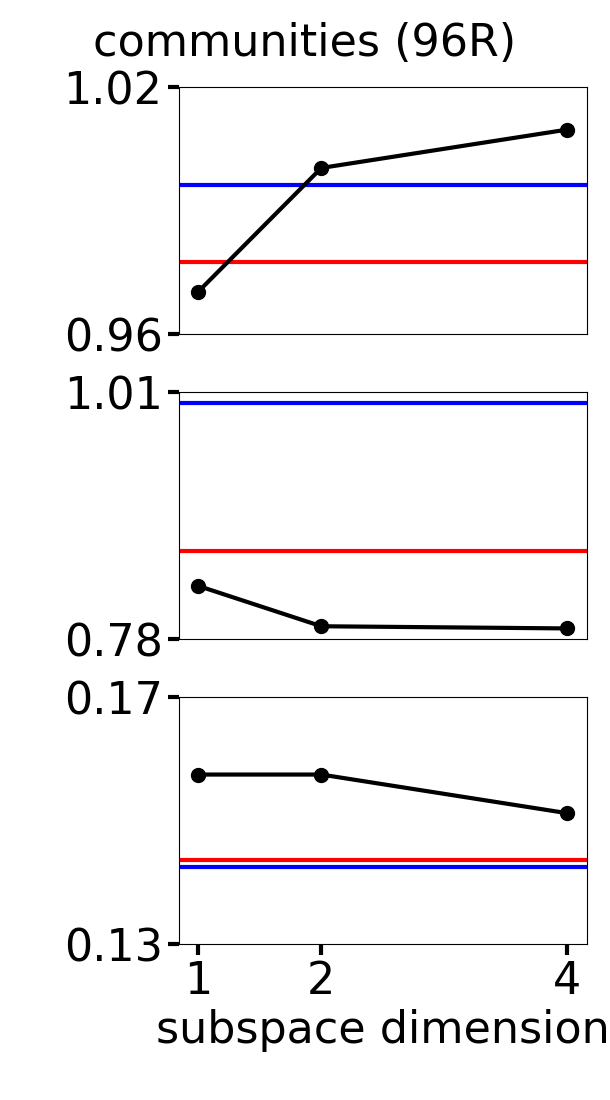}
%\caption{A tiger}\label{fig:tiger}
\end{subfigure}
\begin{subfigure}[b]{.16\linewidth}
\includegraphics[width=\linewidth]{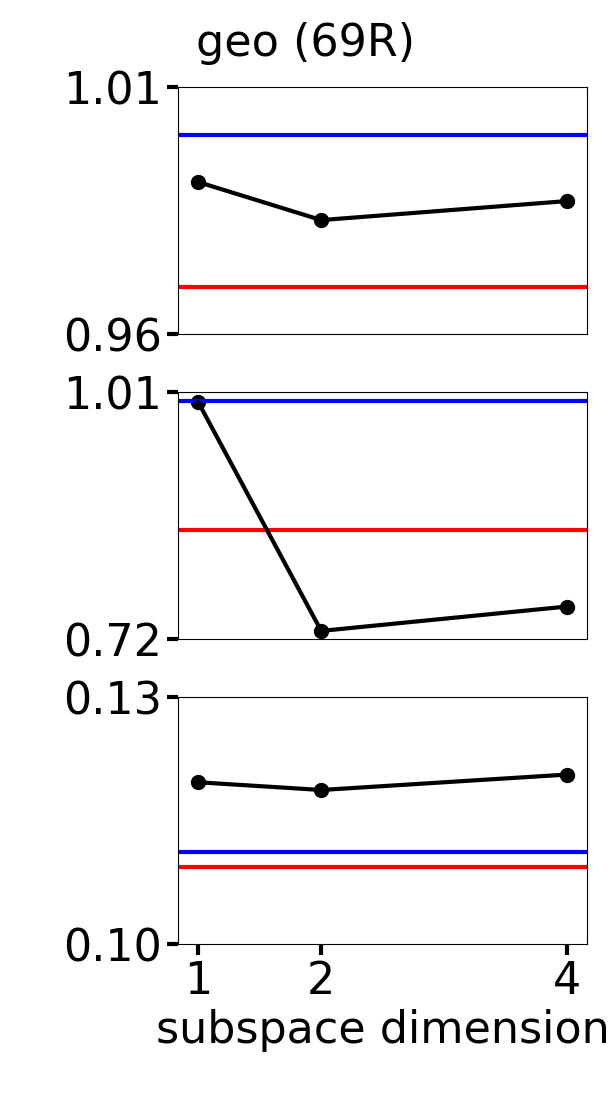}
%\caption{A gull}\label{fig:gull}
\end{subfigure}
\begin{subfigure}[b]{.16\linewidth}
\includegraphics[width=\linewidth]{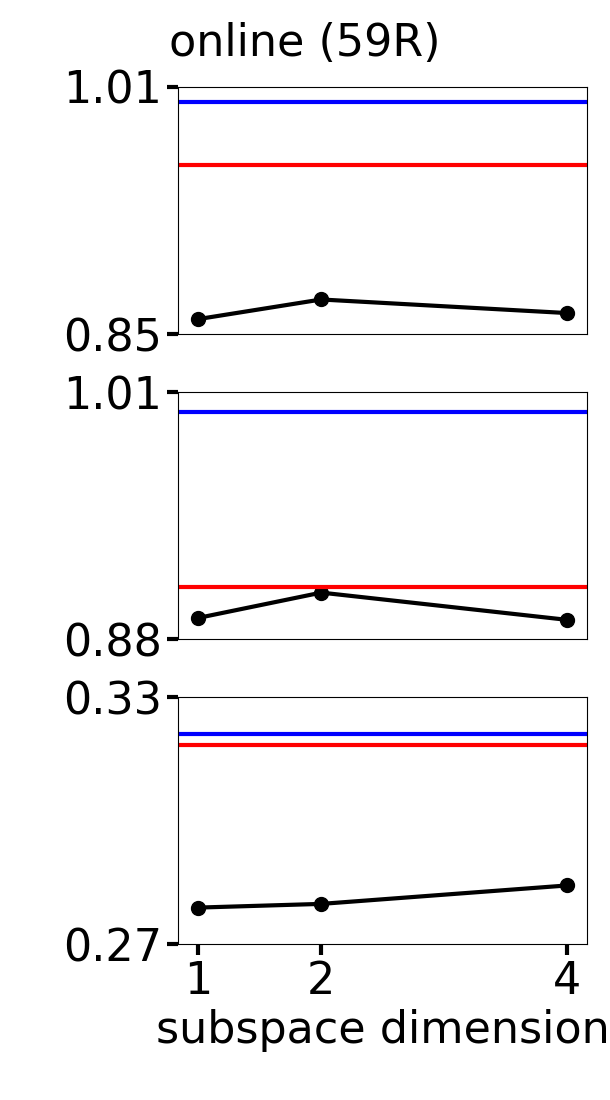}
%\caption{A tiger}\label{fig:tiger}
\end{subfigure}
\begin{subfigure}[b]{.16\linewidth}
\includegraphics[width=\linewidth]{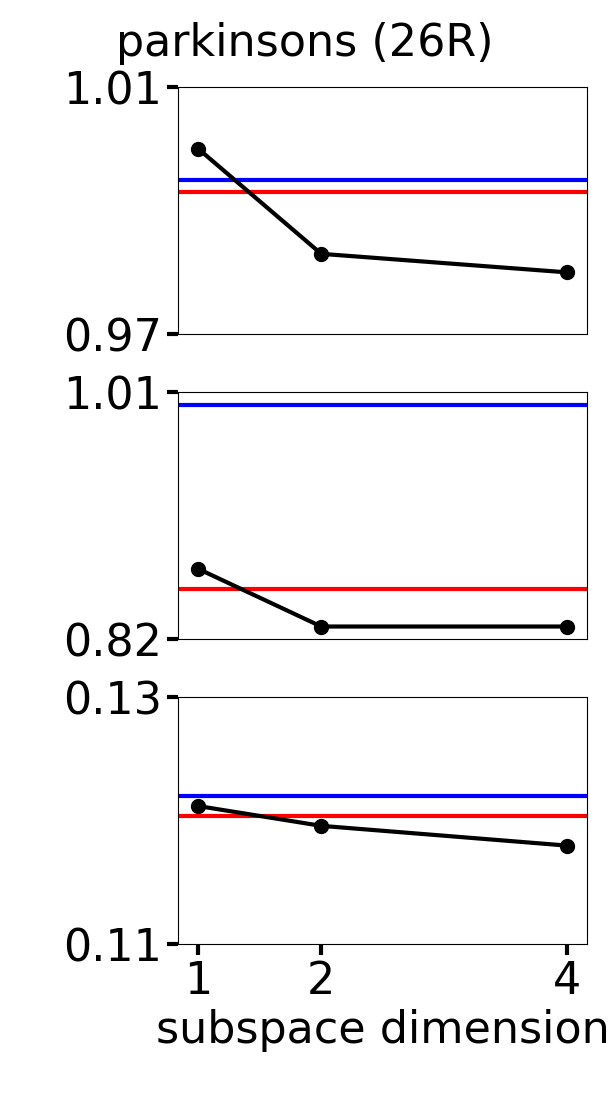}
%\caption{A gull}\label{fig:gull}
\end{subfigure}
\begin{subfigure}[b]{.16\linewidth}
\includegraphics[width=\linewidth]{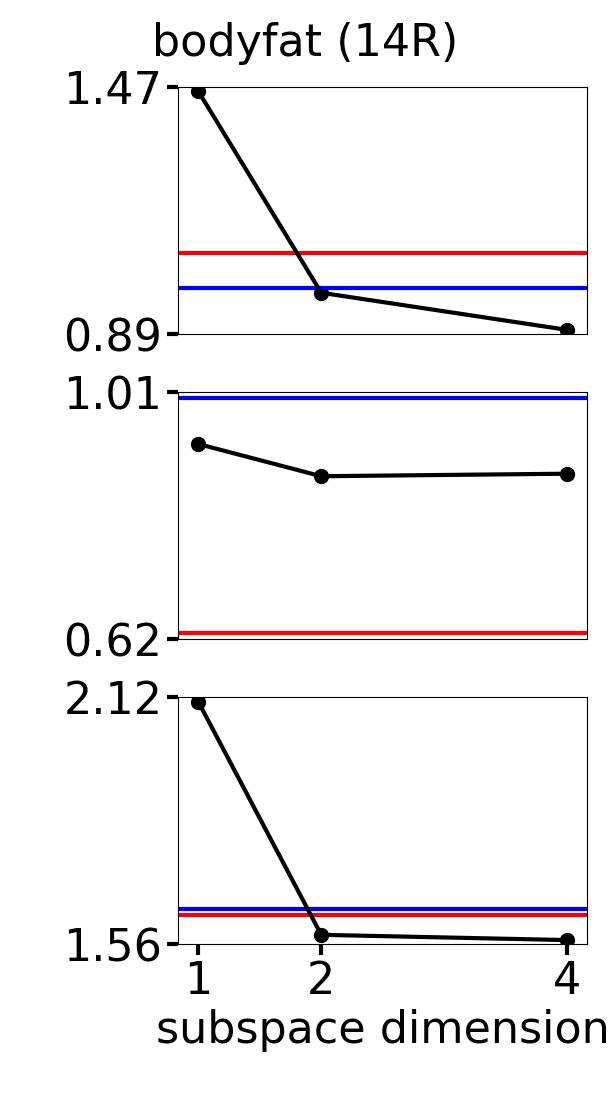}
%\caption{A tiger}\label{fig:tiger}
\end{subfigure}

\caption{For several datasets, the mean loss (``loss''), downstream effective sample size (``N\_eff''), and loss standard deviation (``loss std'') of our method is shown in black as the subspace dimension varies.  The respective values of the unweighted and importance weighting baseline are shown in blue and red, respectively.  Dataset dimension in parentheses.}
\label{fig:real_data_loss}
\end{figure*}

\vspace{-0pt}
\section{Experiments with Real Data}
%We evaluated our method on several classification and regression datasets \cite{Lichman:2013,libsvm}.  Before we describe how we processed the data to create covariate shift and the performance results, we describe how we used cross validation to choose hyperparameters.

\vspace{-0pt}
\subsection{Covariate Shift Experiments}
\subsubsection{Data Preprocessing}
We took several datasets from \cite{Lichman:2013,libsvm}, and for each, introduced covariate shift by creating a sampling scheme that can repeatedly generate training and test data samples.  For each dataset, we first identify a single predictive vector in covariate space by generating 100 random vectors, and measuring the predictive utility of a vector by projecting the covariates onto it, running kernel density regression, and examining the in-sample squared error.  We first sample the dataset to form the training data.  Then to generate the test data, we subsample from the remaining dataset according to the projections along the vector.  Let $t_1,t_0$ denote the max and min projected covariate value in the dataset, and $\sigma$ be the standard deviation of the projected values.  We select data with probability proportional to the density of a $\mathcal{N}(\alpha(t_1-t_o),c\sigma^2)$ distribution based on their projected value.  Values of $\alpha$ close to 0 or 1, and small values of $c$ will lead to a small effective sample size; we choose them so this is achieved.  Given that vector, multiple training and test datasets can then be sampled.  %We shift the data along a predictive direction so that covariate shift adjustment is actually needed.

\vspace{-0pt}
\subsubsection{Experimental Results}

For each dataset, we generated 50 training and test data pairs.  For each pair, we set aside $\tfrac{1}{3}$ of the test data to evaluate out-of-sample performance on the test distribution, and train the model with the remaining data.  Experimental results are in Figure \ref{fig:real_data_loss}, where for each dataset, in 3 subplots we plot in black the respective quantities for our method as the dimension of the projection changes: ``loss'' - the average loss over the 50 replicates, ``N\_eff'' - the average effective sample size enjoyed by the downstream IW estimator, ``loss std'' - the standard deviation in loss over the replicates.  For comparison, we also indicate with horizontal lines the values of those 3 quantities for the unweighted (UW) baseline in blue, and those for the naive importance weighting (IW) baseline in red.  We normalize the results so that UW has a average loss and average effective sample size of 1.  Thus for the loss and N\_eff subplots, there is always a horizontal blue line at a height of 1.  In parentheses are dataset dimension and whether it is classification (C) or regression (R).  We use absolute prediction error loss for regression problems as before, and 0-1 loss for classification problems.

We see that extreme dimension reduction does not always help; whether it does depends on the dataset and the dimension of the subspace.  However, for many datasets, there is some subspace dimension for which our dimension reduction procedure does help.  For a given dataset, the loss can both increase (i.e. communities dataset) or decrease (i.e. facebook) with a reduction in subspace dimension, depending whether the reduction in the variance of the estimator $\hat{L}(b;A,w^A)$, used both in selecting $A$ and the downstream IW, is enough to offset a potential increase in estimator bias.  What dimension reduction offers to the covariate shift problem is a way to navigate that tradeoff.  We also note that as alluded to in Section \ref{sec:analysis} and the simulation study, the variance of our two-step procedure can be higher than  IW.

We also caution that the 3 quantities we report for our method are not directly comparable between different subspace dimensions, because the value of hyperparameter $\lambda$ chosen by cross-validation also differs between dimensions.  This means we cannot make claims like ``the loss standard deviation is lowered when reducing the subspace dimension solely due to decreased density ratio estimation variance''.

%We also offer a caveat in interpreting these results: they do not readily give a way to isolate the impact of reducing the subspace dimension, acting only through the decreased density ratio estimation variance and (expected increase in effective sample size c.f. Lemma \ref{lemma:1}), on the loss estimator variance (reflected in the loss standard deviation).  This is because loss estimator variance is also affected by a reduction in subspace dimension, as acting through the change in hyperparameter $\lambda$, which is chosen via cross-validation, and takes into account other factors affected by subspace dimension, i.e. the potential bias in the estimator.

Despite this caveat, we can still identify a phenomena that is paradoxical at first: when going from 2 dimensions to 1, the utilized effective sample size, N\_eff, tends to increase.  Coupled with the accompanying decrease in density ratio estimation variance, we would expect the loss standard deviation to decrease under this change.  However, we observe exactly the opposite. This phenomena is likely due to the increased number of bad local optima in the non-convex optimization problem when the subspace dimension is extremely low, i.e. 1.  We anecdotally observed the same phenomena in the simulation study, where loss standard deviation was sensitive to how many random restarts we used for optimization.  %Handling this optimization issue is potential future work.

\subsection{Case Study on Learning Subgroup Models}
Recent work has shown that some real world models\cite{pmlr-v81-buolamwini18a} have lower prediction performance on minority subgroups than for the majority group.  In theory, this problem could be alleviated by learning a subgroup model by only using data from the subgroup.  However, as data are scarce, the performance of the subgroup model could be improved by utilizing the appropriate data from the majority group.  In correcting for covariate shift, a model is learned to minimize average prediction error with respect to one population given labelled samples drawn from another population.  This motivates us to recast the subgroup model learning problem as a covariate shift correction problem. We will describe this formulation, then use our method to learn a classification model for depression for the age 18-24 subgroup. %secondly show that our method is able to produce more accurate predictions for the subgroup relative to baselines when data is scarce and the importance weighting estimator's unbiasedness cannot overcome its high variance.

Formally, given $(x_i,y_i,z_i)\sim P_{X,Y,Z}$, where feature vectors $x_i\in \mathbb{R}^D$ and $z_i$ is an indicator denoting membership in the subgroup of interest, given loss function $l$ and model class $\mathcal{F}$, the subgroup model learning problem seeks $\operatorname{argmin}_{f\in\mathcal{F}} E_{P_{X,Y|Z=1}}[l(f(X),Y)]$.  We reformulate it as a covariate shift problem where $P^{te}_{X,Y}=P_{X,Y|Z=1}$ and $P^{tr}_{X,Y}=P_{X,Y}$.  The covariate shift assumption is then that $P_{Y|X,Z=1}=P_{Y|X}$ (which might not actually hold).

\begin{figure}[]
\centering
\begin{subfigure}{0.48\linewidth}
%\centering 
%\includegraphics[trim = 40 200 40 200, width=1\linewidth]{\getFullPath{age_scatters.pdf}}
\includegraphics[width=1\linewidth]{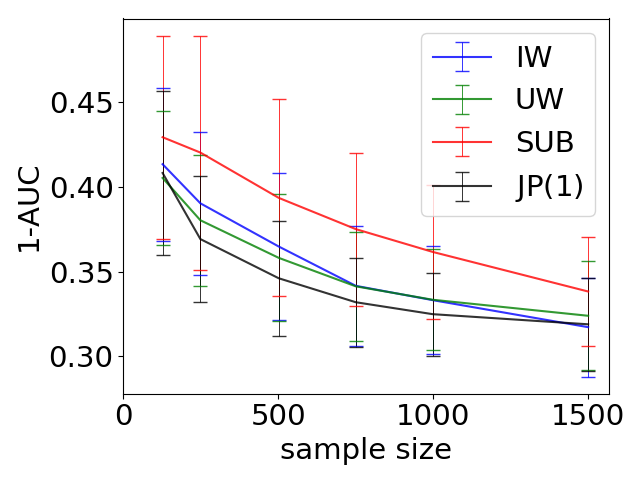}
\caption{Our method has better relative performance when labelled data is scarce.}
\label{fig:nhanes_roc_trend}
\end{subfigure}%~
\begin{subfigure}{0.48\linewidth}
%\centering
\includegraphics[width=1\linewidth]{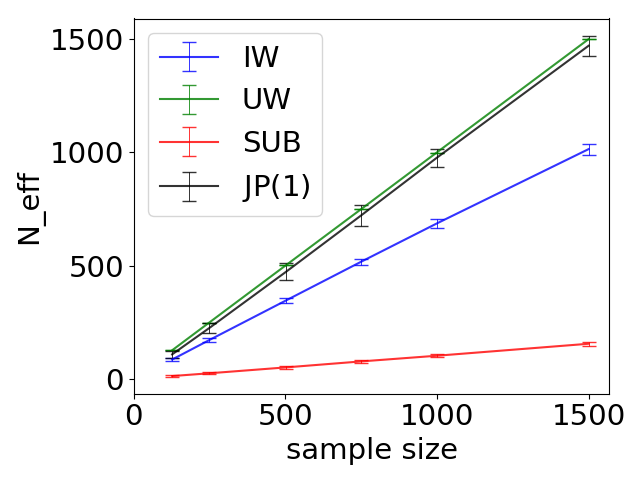}
\caption{Our method is able to use more data to learn the subgroup model than IW \& SUB.}
\label{fig:nhanes_Neff_trend}
\end{subfigure}
\caption{Our method performs relative better when labelled data is scarce due to its ability to utilize more data for fitting.}
\label{fig:nhanes_trends}
\end{figure}

We apply our covariate shift correction method to the National Health and Nutrition Examination Survey (NHANES) \cite{johnson2013national} dataset, downloaded via the eponymous R package \cite{NHANES}.  This dataset contains demographic, physical, health, and lifestyle variables from 10000 individuals in the United States, obtained via surveys between 1999 and 2004.  Importantly, these individuals are sampled to be representative of the United States general population as a whole, so that subgroups underrepresented in the population will also be so in the dataset.  Given the recent rise in mental health issues of college students \cite{depression}, we use our method to learn a classifier to predict the presence of depression for the subgroup of persons of college age, i.e. those between 18 and 24 years of age.  Such persons, according to the NHANES dataset, are on average different from the general population.  For example, they tend to be less likely to smoke, and have sleep troubles, more likely to be single and physically active, and on average have lower weight and poverty levels.  These differences open up the possibility that the best (linear) model for this subgroup may be different than that of the general population.  Furthermore, only about 10\% of the general population (and thus NHANES dataset) falls into this subgroup, so that approaches beyond simply fitting a model to data from the subgroup are needed to handle data scarcity.  

Regarding features, we only use those that can be easily obtained, such as those aforementioned as well as income, pulse, years of education, excluding those require medical tests such as urine sugar level, and those too closely correlated with depression, i.e. how often the individual considers their own general mental health to be bad.  The utility of the classifier would be to identify those at risk of depression, using easily obtainable features that are not proxies for depression.  Following this guideline, we end up using 27 features, and after removing those for which not all features were available, retain data from 4517 individuals, of which 21\% were labelled as being depressed (proportion of days feeling depressed not equal to ``None'').

We compare the performance of our method (projecting to 1-dimension, denoted JP(1)) to the naive importance weighting and unweighted methods (denoted IW and UW in figures) in learning a classifier for depression for the age 18 to 24 subgroup using the covariate shift reformulation, as well as the baseline of fitting the classifier using only data from the subgroup (denoted SUB).  We perform this comparison as the number of labelled data varies - for various $N$, we uniformly subsample a subset of size $N$, train a model using a given method, and obtain predictions for the remaining unsampled data.  We do this 100 times for each $N$ and method, and report the average 1-AUC for them in Figure \ref{fig:nhanes_roc_trend} as well as the average effective sample size (N\_eff) in fitting the models in Figure \ref{fig:nhanes_Neff_trend}.

We see that our method, for all dataset sizes, it is able to consider only (1-dimensional) subspaces in which the effective sample size of the loss estimator is close to that of the entire labelled data without hurting performance (recall the tradeoff controlling the lower bound on effective sample size is chosen via cross validation).  When the labelled data is scarce, the resulting low variance of the loss estimator offsets the bias that results from only estimating density ratios within a subspace, when compared to IW, whose effective sample size is always smaller.  When data is plentiful, i.e. at $N=1500$, the relative advantage of our method in utilizing more data is diminished compared to the baselines, who utilize less biased estimators.  A similar relationship holds between IW and the more biased UW, though the dataset size at which their relative performances change order occurs at $N=1000$.  Finally, SUB, whose estimator is unbiased but enjoys far lower effective sample sizes, has far worse performance than all of the covariate shift approaches, which utilize labelled data from outside the subgroup.

\section{Related Work}

Our work studies the benefit of dimension reduction prior to applying importance weighting approaches \cite{gretton2009covariate,sugiyama2008direct, bickel2009discriminative} for solving the covariate shift problem \cite{shimodaira2000improving}.  In particular, we use dimension reduction as a way to potentially reduce the otherwise high variance in importance weighting approaches pointed out by \cite{cortes2010learning}.  Past work addressing this issue \cite{shimodaira2000improving,cortes2010learning,yamada2011relative,hu2016robust} has used various forms of regularization to discourage exceedingly large weights, at the cost of increased bias.  Another line of work eschews weight estimation and builds predictors that are robust to potential shifts in the covariate \cite{wen2014robust} or conditional outcome \cite{liu2014robust,chen2016robust} distribution.  A separate work \cite{reddi2015doubly} reduces variance by using the predictor minimizing training loss as a prior when minimizing the reweighted loss.  One work \cite{sugiyama2010direct} does apply dimension reduction before estimating density ratios, by searching for a subspace in which the training and test densities are maximally different, and estimating the ratios within it.  However, their chosen subspace, by construction, would result in a small underlying effective sample size relative to other subspaces.

Our approach is similar to supervised representation learning methods that find representations that are both predictive and similar in some sense between the two domains; penalizing subspaces with small effective sample sizes is equivalent to penalizing subspaces in which the Pearson divergence is high.  However, none of those approaches perform importance weighting when evaluating the predictive utility of a representation, which is advantageous in many situations when the downstream model (linear model in our case) is misspecified so that importance weighting is actually necessary to begin with.  For example, \cite{chen2009extracting} finds a linear subspace minimizing an additive combination of maximum mean discrepancy, and the lowest possible \emph{unweighted} squared loss of a linear model acting on it.  Recent work has used neural networks to find representations with which low unweighted training domain loss can be achieved, and which are similar between domains, as accomplished through using closed form discrepancy measures \cite{johansson2016learning} or adversarial training \cite{ganin2016domain}.  However, the learned model is not linear and interpretable.

\section{Conclusion}
To address the high variance of importance weighting (IW) approaches to handling covariate shift, we have explored the benefit of extreme dimension reduction, i.e. to very low dimensions, prior to applying such approaches.  We presented and solved a bilevel optimization problem that searches for a subspace that is predictive, and for which the downstream IW procedure would enjoy a large effective sample size.  We illustrated through lemmas and simulated and real data that dimension reduction helps sometimes but not always, due to bias introduced into the IW loss estimator, and the fact that we incur additional variance in the selection of the subspace.  Study into how to alleviate these issues, as well as how to avoid bad local optima in our optimization problem, are potential future work.  We also formulated the problem of learning a subgroup-specific model as a covariate shift problem, and demonstrate the advantage of dimension reduction in learning a classifier for depression for the college-aged subgroup.  Finally, we believe our work has implications for causal inference, which can be formulated as a covariate shift problem \cite{johansson2016learning}, and where the propensity score, essentially a density ratio, is used widely \cite{austin2011introduction}, yet unreliably estimated in higher dimensions \cite{kang2007demystifying}.

%presented a method that jointly learns a subspace upon which predictions depend, and within which density ratios are estimated.  Our method lowers variance through estimating ratios in a reduced space, as well as a regularizer discouraging subspaces with low effective sample sizes whose tradeoff parameter can be cross-validated.  To fit our model, we leverage past work on hyperparameter selection via gradient descent.  We explain through lemmas why the variance is reduced and the extent of our test loss estimator's bias, and then show empirically that in many real data, sacrificing bias for reduced variance helps performance.  Finally, we believe our work has implications for causal inference, which can be formulated as a covariate shift problem \cite{johansson2016learning}, and where the propensity score, essentially a density ratio, is used widely \cite{austin2011introduction}, yet unreliably estimated in higher dimensions \cite{kang2007demystifying}.

%\clearpage
\bibliographystyle{plain}
\bibliography{kernel_notes}

\begin{thebibliography}{10}

\bibitem{libsvm}
Libsvm datasets.
\newblock \url{http://www.csie.ntu.edu.tw/~cjlin/libsvmtools/datasets/}.

\bibitem{amos2017optnet}
Brandon Amos and J~Zico Kolter.
\newblock Optnet: Differentiable optimization as a layer in neural networks.
\newblock {\em arXiv preprint arXiv:1703.00443}, 2017.

\bibitem{depression}
American~Psychological Association.
\newblock {\em The Crisis on Campus}, 2011.

\bibitem{austin2011introduction}
Peter~C Austin.
\newblock An introduction to propensity score methods for reducing the effects
  of confounding in observational studies.
\newblock {\em Multivariate behavioral research}, 46(3):399--424, 2011.

\bibitem{bickel2009discriminative}
Steffen Bickel, Michael Br{\"u}ckner, and Tobias Scheffer.
\newblock Discriminative learning under covariate shift.
\newblock {\em Journal of Machine Learning Research}, 10(Sep):2137--2155, 2009.

\bibitem{pmlr-v81-buolamwini18a}
Joy Buolamwini and Timnit Gebru.
\newblock Gender shades: Intersectional accuracy disparities in commercial
  gender classification.
\newblock In Sorelle~A. Friedler and Christo Wilson, editors, {\em Proceedings
  of the 1st Conference on Fairness, Accountability and Transparency},
  volume~81 of {\em Proceedings of Machine Learning Research}, pages 77--91,
  New York, NY, USA, 23--24 Feb 2018. PMLR.

\bibitem{chen2009extracting}
Bo~Chen, Wai Lam, Ivor Tsang, and Tak-Lam Wong.
\newblock Extracting discriminative concepts for domain adaptation in text
  mining.
\newblock In {\em Proceedings of the 15th ACM SIGKDD international conference
  on Knowledge discovery and data mining}, pages 179--188. ACM, 2009.

\bibitem{chen2016robust}
Xiangli Chen, Mathew Monfort, Anqi Liu, and Brian~D Ziebart.
\newblock Robust covariate shift regression.
\newblock In {\em Artificial Intelligence and Statistics}, pages 1270--1279,
  2016.

\bibitem{cortes2010learning}
Corinna Cortes, Yishay Mansour, and Mehryar Mohri.
\newblock Learning bounds for importance weighting.
\newblock In {\em Advances in neural information processing systems}, pages
  442--450, 2010.

\bibitem{doucet2001sequential}
Arnaud Doucet, Nando De~Freitas, and NJ~Gordon.
\newblock Sequential monte carlo methods in practice. series statistics for
  engineering and information science, 2001.

\bibitem{foo2008efficient}
Chuan-sheng Foo, Chuong~B Do, and Andrew~Y Ng.
\newblock Efficient multiple hyperparameter learning for log-linear models.
\newblock In {\em Advances in neural information processing systems}, pages
  377--384, 2008.

\bibitem{ganin2016domain}
Yaroslav Ganin, Evgeniya Ustinova, Hana Ajakan, Pascal Germain, Hugo
  Larochelle, Fran{\c{c}}ois Laviolette, Mario Marchand, and Victor Lempitsky.
\newblock Domain-adversarial training of neural networks.
\newblock {\em Journal of Machine Learning Research}, 17(59):1--35, 2016.

\bibitem{gretton2009covariate}
Arthur Gretton, Alex Smola, Jiayuan Huang, Marcel Schmittfull, Karsten
  Borgwardt, and Bernhard Sch{\"o}lkopf.
\newblock Covariate shift by kernel mean matching.
\newblock 2009.

\bibitem{hu2016robust}
Weihua Hu, Issei Sato, and Masashi Sugiyama.
\newblock Robust supervised learning under uncertainty in dataset shift.
\newblock {\em arXiv preprint arXiv:1611.02041}, 2016.

\bibitem{johansson2016learning}
Fredrik Johansson, Uri Shalit, and David Sontag.
\newblock Learning representations for counterfactual inference.
\newblock In {\em International Conference on Machine Learning}, pages
  3020--3029, 2016.

\bibitem{johnson2013national}
Clifford~L Johnson, Ryne Paulose-Ram, Cynthia~L Ogden, Margaret~D Carroll,
  Deanna Kruszan-Moran, Sylvia~M Dohrmann, and Lester~R Curtin.
\newblock National health and nutrition examination survey. analytic
  guidelines, 1999-2010.
\newblock 2013.

\bibitem{kanamori2009least}
Takafumi Kanamori, Shohei Hido, and Masashi Sugiyama.
\newblock A least-squares approach to direct importance estimation.
\newblock {\em Journal of Machine Learning Research}, 10(Jul):1391--1445, 2009.

\bibitem{kang2007demystifying}
Joseph~DY Kang and Joseph~L Schafer.
\newblock Demystifying double robustness: A comparison of alternative
  strategies for estimating a population mean from incomplete data.
\newblock {\em Statistical science}, pages 523--539, 2007.

\bibitem{li1991sliced}
Ker-Chau Li.
\newblock Sliced inverse regression for dimension reduction.
\newblock {\em Journal of the American Statistical Association},
  86(414):316--327, 1991.

\bibitem{Lichman:2013}
M.~Lichman.
\newblock {UCI} machine learning repository, 2013.

\bibitem{liu2014robust}
Anqi Liu and Brian Ziebart.
\newblock Robust classification under sample selection bias.
\newblock In {\em Advances in neural information processing systems}, pages
  37--45, 2014.

\bibitem{maclaurin2015gradient}
Dougal Maclaurin, David Duvenaud, and Ryan Adams.
\newblock Gradient-based hyperparameter optimization through reversible
  learning.
\newblock In {\em International Conference on Machine Learning}, pages
  2113--2122, 2015.

\bibitem{maclaurin2015autograd}
Dougal Maclaurin, David Duvenaud, and Ryan~P Adams.
\newblock Autograd: Reverse-mode differentiation of native python.
\newblock In {\em ICML workshop on Automatic Machine Learning}, 2015.

\bibitem{NHANES}
Randall Pruim.
\newblock {\em NHANES: Data from the US National Health and Nutrition
  Examination Study}, 2015.
\newblock R package version 2.1.0.

\bibitem{reddi2015doubly}
Sashank~Jakkam Reddi, Barnabas Poczos, and Alexander~J Smola.
\newblock Doubly robust covariate shift correction.
\newblock 2015.

\bibitem{shimodaira2000improving}
Hidetoshi Shimodaira.
\newblock Improving predictive inference under covariate shift by weighting the
  log-likelihood function.
\newblock {\em Journal of statistical planning and inference}, 90(2):227--244,
  2000.

\bibitem{sugiyama2010direct}
Masashi Sugiyama, Satoshi Hara, Paul Von~B{\"u}nau, Taiji Suzuki, Takafumi
  Kanamori, and Motoaki Kawanabe.
\newblock Direct density ratio estimation with dimensionality reduction.
\newblock In {\em Proceedings of the 2010 SIAM International Conference on Data
  Mining}, pages 595--606. SIAM, 2010.

\bibitem{sugiyama2007covariate}
Masashi Sugiyama, Matthias Krauledat, and Klaus-Robert Muller.
\newblock Covariate shift adaptation by importance weighted cross validation.
\newblock {\em Journal of Machine Learning Research}, 8(May):985--1005, 2007.

\bibitem{sugiyama2008direct}
Masashi Sugiyama, Shinichi Nakajima, Hisashi Kashima, Paul~V Buenau, and
  Motoaki Kawanabe.
\newblock Direct importance estimation with model selection and its application
  to covariate shift adaptation.
\newblock In {\em Advances in neural information processing systems}, pages
  1433--1440, 2008.

\bibitem{JMLR:v17:16-177}
James Townsend, Niklas Koep, and Sebastian Weichwald.
\newblock Pymanopt: A python toolbox for optimization on manifolds using
  automatic differentiation.
\newblock {\em Journal of Machine Learning Research}, 17(137):1--5, 2016.

\bibitem{wen2014robust}
Junfeng Wen, Chun-nam Yu, and Russell Greiner.
\newblock Robust learning under uncertain test distributions: Relating
  covariate shift to model misspecification.
\newblock In {\em Proceedings of the 31st International Conference on Machine
  Learning (ICML-14)}, pages 631--639, 2014.

\bibitem{yamada2011relative}
Makoto Yamada, Taiji Suzuki, Takafumi Kanamori, Hirotaka Hachiya, and Masashi
  Sugiyama.
\newblock Relative density-ratio estimation for robust distribution comparison.
\newblock In {\em Advances in neural information processing systems}, pages
  594--602, 2011.

\end{thebibliography}



\section*{Supplementary material}

# Supplementary Material

## Choosing Hyperparameters

The hyperparameters in the formulation fall into two groups: those whose optimal value does not change with $A$, namely $\lambda$, and those whose optimal value does change with $A$: $c$, which prevents overfitting when evaluating the predictive utility of $A$, and those required for density ratio estimation, namely $\gamma$ as well as the bandwidth of the Gaussian kernel basis functions, which we refer to as $\sigma$. We do not consider the latter group of hyperparameters fixed, but instead update them throughout the gradient descent over $A$. We also note that by updating these hyperparameters as the optimization runs, they are no longer hyperparameters in the conventional sense. The reduced number of hyperparameters is thus computationally convenient.

The optimal values of $\gamma$ and $\sigma$ change with $A$ because they are used to obtain estimates $\hat{w}_i^A$ of which changes with $A$. Certainly, the hyperparameters of a density ratio estimation procedure should change as the true density ratios change. To choose $\gamma, \sigma$ for a given $A$, we use exactly the cross validation procedure [1] suggest, with 3-fold cross-validation.

The optimal value of $c$ is presumed to change with $A$ for similar reasons. We would like to choose $c$ such that given $A$, $b^*$, obtained in Equation 7 (i.e. IW loss minimization, regularized by $c$), equals $\mathrm{argmin}_{b \in \mathbb{R}^K} E_{P^{te}}[l(b^T A^T X)]$, i.e. the best linear model for the test domain when covariates are first projected by $A$. Since $c$ is the regularization needed to solve a $A$-dependent learning problem, $c$ should depend on $A$. To evaluate out-of-sample test domain performance for a given $A$ and linear predictor $b$, we use 3-fold *weighted* cross-validation as described in [2], where a weighted average of out-of-sample labelled training data samples is used to approximate test domain loss. A key difference is that given $A$, we estimate density ratios of the covariates, *projected* by $A$, as the density ratios in the full space would not be estimated reliably.

While $\sigma, \gamma$ are updated throughout a single optimization run, $\lambda$ is held fixed, and thus chosen by "traditional" out-of-sample validation. In particular, to evaluate a given $\lambda$, we use 5-fold weighted cross-validation, similar to the previous section; using in-sample labelled training domain and unlabelled test domain data, we obtain an $A$ using the gradient descent procedure, with $c, \gamma, \sigma$ updated through out it. The procedure returns an $A$, and we use that $A$ to calculate the density ratios of *projected* out-of-sample training data, which we then use to approximate test domain loss.

In summary, in $\lambda$, the only "conventional" hyperparameter, is chosen to be the value whose resultant predictors have the lowest average value of a loss estimate computed using the validation sets from 5-fold cross-validation and estimates of projected density ratios. The remaining hyperparameters $\gamma, \sigma, c$ are updated throughout the optimization procedure for a fixed $\lambda$ using grid search 3-fold cross-validation.

## Proof of Lemma 3.1.2

**Lemma.** *Given $A, b$, let $C$ be a matrix whose columns span the orthogonal complement to the subspace spanned by the columns of $A$, and let $(U, V) = (A^T X, C^T X)$ and $(x_i^{tr}, y_i^{tr}) \sim P_{X,Y}^{tr}$. Then*

$$\left| E_{P_{X,Y}^{tr}}[\hat{L}(b; A, w^A)] - E_{P_{X,Y}^{te}}[l(b^T A^T X, Y)] \right| \leq E_{P_U^{tr}}[PE(P_{V|U}^{te} \| P_{V|U}^{tr})]^{\frac{1}{2}} E_{P_U^{te}}[Var_{P_{Y|U}^{tr}}(l(b^T U, Y))]^{\frac{1}{2}},$$

1



\end{document}